\documentclass{article}

\usepackage[accepted]{icml2023}

\usepackage{microtype}
\usepackage{graphicx}
\usepackage{booktabs}

\usepackage{hyperref}

\usepackage{amsmath}
\usepackage{amssymb}
\usepackage{mathtools}
\usepackage{amsthm}

\theoremstyle{plain}
\newtheorem{theorem}{Theorem}[section]

\theoremstyle{definition}
\newtheorem{definition}[theorem]{Definition}

\theoremstyle{remark}

\usepackage{graphicx}
\usepackage[utf8]{inputenc} 
\usepackage[T1]{fontenc}    
\usepackage{url}            
\usepackage{booktabs}       
\usepackage{amsfonts}       
\usepackage{nicefrac}       
\usepackage{url}            
\usepackage{booktabs}       
\usepackage{amsfonts}       
\usepackage{nicefrac}       
\usepackage{microtype}      
\usepackage{nccmath}
\usepackage{times}
\usepackage{graphicx}
\usepackage{natbib}
\usepackage{algorithm}
\usepackage{algorithmic}
\usepackage{amsfonts}
\usepackage{amsmath}

\usepackage{mathtools}
\usepackage{amsthm}
\usepackage{graphicx}
\usepackage{enumerate}
\usepackage{algorithm}
\usepackage{algorithmic}
\usepackage{thm-restate}

\usepackage{url}
\usepackage{dsfont}
\usepackage[mathscr]{euscript}
\usepackage{booktabs}
\usepackage{makecell}
\usepackage{tabularx}
\usepackage{xcolor, colortbl}
\usepackage[normalem]{ulem}
\usepackage{soul}
\usepackage{prettyref}

\usepackage{multirow}
\usepackage{diagbox}

\usepackage{MnSymbol}
\DeclareMathAlphabet\mathbb{U}{msb}{m}{n}
\usepackage{xpatch}

\def\Rset{\mathbb{R}}

\DeclareMathOperator*{\E}{\mathbb E}
\DeclareMathOperator*{\argmax}{argmax}

\DeclarePairedDelimiter{\abs}{\lvert}{\rvert} 
\DeclarePairedDelimiter{\bracket}{[}{]}
\DeclarePairedDelimiter{\curl}{\{}{\}}
\DeclarePairedDelimiter{\norm}{\|}{\|}
\DeclarePairedDelimiter{\paren}{(}{)}

\newcommand{\sA}{{\mathscr A}}

\newcommand{\sC}{{\mathscr C}}
\newcommand{\sD}{{\mathscr D}}
\newcommand{\sF}{{\mathscr F}}

\newcommand{\sH}{{\mathscr H}}

\newcommand{\sM}{{\mathscr M}}
\newcommand{\sP}{{\mathscr P}}

\newcommand{\sR}{{\mathscr R}}
\newcommand{\sS}{{\mathscr S}}
\newcommand{\sT}{{\mathscr T}}

\newcommand{\sX}{{\mathscr X}}
\newcommand{\sY}{{\mathscr Y}}

\newcommand{\sfB}{{\mathsf B}}

\newcommand{\Rad}{\mathfrak R}

\newcommand{\h}{\widehat}
\newcommand{\ov}{\overline}

\newcommand{\wt}{\widetilde}
\newcommand{\e}{\epsilon}
\newcommand{\set}[2][]{#1 \{ #2 #1 \} }
\newcommand{\ignore}[1]{}

\newcommand{\hh}{{\sf h}}
\newcommand{\NA}{---}

\hypersetup{
  colorlinks   = true,
  urlcolor     = blue,
  linkcolor    = blue,
  citecolor    = blue
}

\usepackage[disable,textsize=tiny]{todonotes}

\usepackage[toc, page, header]{appendix}
\icmltitlerunning{Cross-Entropy Loss Functions:
    Theoretical Analysis and Applications}

\begin{document}

\twocolumn[
  \icmltitle{Cross-Entropy Loss Functions:\\
    Theoretical Analysis and Applications}

\begin{icmlauthorlist}
\icmlauthor{Anqi Mao}{courant}
\icmlauthor{Mehryar Mohri}{google,courant}
\icmlauthor{Yutao Zhong}{courant}
\end{icmlauthorlist}

\icmlaffiliation{courant}{Courant Institute of Mathematical Sciences, New York, NY;}
\icmlaffiliation{google}{Google Research, New York, NY}

\icmlcorrespondingauthor{Anqi Mao}{aqmao@cims.nyu.edu}
\icmlcorrespondingauthor{Mehryar Mohri}{mohri@google.com}
\icmlcorrespondingauthor{Yutao Zhong}{yutao@cims.nyu.edu}

\icmlkeywords{learning theory, surrogate loss, consistency, adversarial robustness}

\vskip 0.3in
]

\printAffiliationsAndNotice{}

\begin{abstract}
Cross-entropy is a widely used loss function in applications. It
coincides with the logistic loss applied to the outputs of a neural
network, when the softmax is used. But, what guarantees can we rely on
when using cross-entropy as a surrogate loss? We present a theoretical
analysis of a broad family of loss functions, \emph{comp-sum losses},
that includes cross-entropy (or logistic loss), generalized
cross-entropy, the mean absolute error and other 
cross-entropy-like loss functions.
We give the first $\sH$-consistency bounds for these loss functions.
These are non-asymptotic guarantees that upper bound the zero-one loss
estimation error in terms of the estimation error of a surrogate loss,
for the specific hypothesis set $\sH$ used. We further show that our
bounds are \emph{tight}. These bounds depend on quantities called
\emph{minimizability gaps}. To make them more explicit, we give a
specific analysis of these gaps for comp-sum losses.
We also introduce a new family of loss functions, \emph{smooth
adversarial comp-sum losses}, that are derived from their comp-sum
counterparts by adding in a related smooth term.  We show that these
loss functions are beneficial in the adversarial setting by proving
that they admit $\sH$-consistency bounds.  This leads to new
adversarial robustness algorithms that consist of minimizing a
regularized smooth adversarial comp-sum loss.
While our main purpose is a theoretical analysis, we also present an
extensive empirical analysis comparing comp-sum losses. We further
report the results of a series of experiments demonstrating that our
adversarial robustness algorithms outperform the current
state-of-the-art, while also achieving a superior non-adversarial
accuracy.
\end{abstract}

\section{Introduction}
\label{sec:intro}

Most current learning algorithms rely on minimizing the cross-entropy
loss to achieve a good performance in a classification task. This is
because directly minimizing the zero-one classification loss is
computationally hard. But, what guarantees can we benefit from when
using cross-entropy as a surrogate loss?

Cross-entropy coincides with the (multinomial) logistic loss applied to the
outputs of a neural network, when the
softmax is used. It is known that the
logistic loss is Bayes consistent \citep{Zhang2003}. Thus,
asymptotically, a nearly optimal minimizer of the logistic loss over
the family of all measurable functions is also a nearly optimal
optimizer of the zero-one classification loss. However, this does not
supply any information about learning with a typically restricted
hypothesis set, which of course would not contain all measurable
functions. It also provides no guarantee for approximate minimizers
(non-asymptotic guarantee) since convergence could be arbitrarily
slow. What non-asymptotic guarantees can we rely on when minimizing
the logistic loss with a restricted hypothesis set, such as a family
of neural networks?

Recent work by
\citet*{awasthi2022Hconsistency,AwasthiMaoMohriZhong2022multi}
introduced the notion of \emph{$\sH$-consistency bounds}. These are
upper bounds on the zero-one estimation error of any predictor in a
hypothesis set $\sH$ in terms of its surrogate loss estimation error.
Such guarantees are thus both non-asymptotic and hypothesis
set-specific and therefore more informative than Bayes consistency
guarantees
\citep{Zhang2003,bartlett2006convexity,steinwart2007compare,
  tewari2007consistency}. For multi-class classification, the authors
derived such guarantees for the so-called \emph{sum-losses}, such as
the sum-exponential loss function of \citet{weston1998multi}, and
\emph{constrained losses}, such as the loss function of
\citet{lee2004multicategory}, where the scores (logits) must sum to
zero. To the best of our knowledge, no such guarantee has been given
for the more widely used logistic loss (or cross-entropy).

This paper presents the first $\sH$-consistency bounds for the
logistic loss, which can be used to derive directly guarantees for
current algorithms used in the machine learning community.  More
generally, we will consider a broader family of loss functions that we
refer to as \emph{comp-sum losses}, that is loss functions obtained by
composition of a concave function, such as logarithm in the case of
the logistic loss, with a sum of functions of differences of score,
such as the negative exponential. We prove $\sH$-consistency bounds
for a wide family of comp-sum losses, which includes as special cases
the logistic loss
\citep{Verhulst1838,Verhulst1845,Berkson1944,Berkson1951}, the
\emph{generalized cross-entropy loss} \citep{zhang2018generalized},
and the \emph{mean absolute error loss} \citep{ghosh2017robust}.
We further show that our bounds are \emph{tight} and thus cannot be
improved.

$\sH$-consistency bounds are expressed in terms of a quantity called
\emph{minimizability gap}, which only depends on the loss function and
the hypothesis set $\sH$ used. It is the difference of the best-in
class expected loss and the expected pointwise infimum of the loss.
For the loss functions we consider, the minimizability gap vanishes
when $\sH$ is the full family of measurable functions. However, in
general, the gap is non-zero and plays an important role, depending on
the property of the loss function and the hypothesis set. Thus, to
better understand $\sH$-consistency bounds for comp-sum losses, we
specifically analyze their minimizability gaps, which we
use to compare their guarantees.

A recent challenge in the application of neural networks is their
robustness to imperceptible perturbations
\citep{szegedy2013intriguing}. While neural networks trained on large
datasets often achieve a remarkable performance
\citep{SutskeverVinyalsLe2014,KrizhevskySutskeverHinton2012}, their
accuracy remains substantially lower in the presence of such
perturbations. One key issue in this scenario is the definition of a
useful surrogate loss for the adversarial loss. To tackle this
problem, we introduce a family of loss functions designed for
adversarial robustness that we call \emph{smooth adversarial comp-sum
loss functions}. These are loss functions derived from their comp-sum
counterparts by augmenting them with a natural smooth term.  We show
that these loss functions are beneficial in the adversarial setting by
proving that they admit $\sH$-consistency bounds. This leads to a
family of algorithms for adversarial robustness that consist of
minimizing a regularized smooth adversarial comp-sum loss.

While our main purpose is a theoretical analysis, we also present an
extensive empirical analysis. We compare the empirical performance of
comp-sum losses for different tasks and relate that to their
theoretical properties. We further report the results of experiments
with the CIFAR-10, CIFAR-100 and SVHN datasets comparing the
performance of our algorithms based on smooth adversarial comp-sum
losses with that of the state-of-the-art algorithm for this task
\textsc{trades} \citep{zhang2019theoretically}. The results show that
our adversarial algorithms outperform \textsc{trades} and also achieve
a substantially better non-adversarial (clean) accuracy.

The rest of this paper is organized as follows. In
Section~\ref{sec:preliminaries}, we introduce some basic concepts and
definitions related to comp-sum loss functions.  In
Section~\ref{sec:H-consistency-bounds}, we present our
$\sH$-consistency bounds for comp-sum losses.  We further carefully
compare their minimizability gaps in Section~\ref{sec:comparison}. In
Section~\ref{sec:adversarial}, we define and motivate our smooth
adversarial comp-sum losses, for which we prove $\sH$-consistency
bounds, and briefly discuss corresponding adversarial algorithms. In
Section~\ref{sec:experiments}, we report the results of our
experiments both to compare comp-sum losses in several tasks, and to
compare the performance of our algorithms based on smooth adversarial
comp-sum losses. In Section~\ref{sec:future-work}, we discuss avenues
for future research. In Appendix~\ref{app:realted-work}, we give a
comprehensive discussion of related work.

\section{Preliminaries}
\label{sec:preliminaries}

We consider the familiar multi-class classification
setting and denote by $\sX$ the input space,
by $\sY = [n] = \set{1, \ldots, n}$ the set of 
classes or categories ($n \geq 2$) and by $\sD$ a distribution
over $\sX \times \sY$.

We study general loss functions $\ell\colon \sH_{\rm{all}} \times \sX
\times \sY \to \Rset$ where $\sH_{\rm{all}}$ is the family of all
measurable functions $h\colon \sX \times \sY \to \Rset$. In
particular, the zero-one classification loss is defined, for all $h
\in \sH_{\rm{all}}$, $x \in \sX$ and $y \in \sY$, by $\ell_{0-1}(h, x,
y) = 1_{\hh(x) \neq y}$, where $\hh(x) = \argmax_{y \in \sY}h(x, y)$
with an arbitrary but fixed deterministic strategy used for breaking
the ties. For simplicity, we fix that strategy to be
the one selecting the label with the highest index under the natural
ordering of labels.

We denote by $\sR_\ell(h)$ the generalization error or expected loss
of a hypothesis $h \colon \sX \times \sY \to \Rset$: $\sR_\ell(h) =
\E_{(x, y) \sim \sD}[\ell(h, x, y)]$. For a hypothesis set $\sH
\subseteq \sH_{\rm{all}}$ of functions mapping from $\sX \times \sY$
to $\Rset$, $\sR^*_\ell(\sH)$ denotes the best-in class expected loss:
$\sR^*_\ell(\sH) = \inf_{h \in \sH} \sR_\ell(h)$.

We will prove $\sH$-consistency bounds, which are inequalities
relating the zero-one classification estimation loss $\ell_{0-1}$ of
any hypothesis $h \in \sH$ to that of its surrogate loss $\ell$
\citep*{awasthi2022Hconsistency,AwasthiMaoMohriZhong2022multi}. They
take the following form: $\forall h \in \sH, \sR_{\ell_{0-1}}(h) -
\sR^*_{\ell_{0-1}}(\sH) \leq f(\sR_\ell(h) - \sR^*_\ell(\sH))$, where
$f$ is a non-decreasing real-valued function. Thus, they show that the
estimation zero-one loss of $h$, $\sR_{\ell_{0-1}}(h) -
\sR^*_{\ell_{0-1}}(\sH)$, is bounded by $f(\e)$ when its surrogate
estimation loss, $\sR_\ell(h) - \sR^*_\ell(\sH)$, is bounded by
$\e$. These guarantees are thus non-asymptotic and depend on the
hypothesis set $\sH$ considered.

$\sH$-consistency bounds are expressed in terms of a quantity
depending on the hypothesis set $\sH$ and the loss function $\ell$
called \emph{minimizability gap}, and defined by $\sM_\ell(\sH) =
\sR^*_{\ell}(\sH) - \E_x\bracket[\big]{\inf_{h \in \sH} \E_y
  \bracket*{\ell(h, X, y) \mid X = x}}$. By the super-additivity of
the infimum, since $\sR^*_{\ell}(\sH) = \inf_{h \in \sH}
\E_x\bracket[\big]{\E_y \bracket*{\ell(h, X, y) \mid X = x}}$, the
minimizability gap is always non-negative.  It measures the difference
between the best-in-class expected loss and the expected infimum of
the pointwise expected loss.  When the loss function $\ell$ only
depends on $h(x,\cdot)$ for all $h$, $x$, and $y$, that is $\ell(h, x,
y) = \Psi(h(x,1), \ldots, h(x,n), y)$, for some function $\Psi$, then
it is not hard to show that the minimizability gap vanishes for the
family of all measurable functions: $\sM(\sH_{\rm{all}}) = 0$
\citep{steinwart2007compare}[lemma~2.5].  In general, however, the
minimizabiliy gap is non-zero for a restricted hypothesis set $\sH$
and is therefore important to analyze.  Note that the minimizabiliy
gap can be upper bounded by the approximation error $\sA(\sH) =
\sR^*_{\ell}(\sH) - \E_x\bracket[\big]{\inf_{h \in \sH_{\rm{all}}}
  \E_y \bracket*{\ell(h, X, y) \mid X = x}}$.  It is however a finer
quantity than the approximation error and can thus lead to more
favorable guarantees.

\ignore{
  This holds for many loss functions used in practice but typically
  not for loss functions used for adversarial robustness.
}

\ignore{
— definition of comp-sum losses.
— notions of calibration function C and minimizability gaps.
— definition of H-cons bounds.
— desired H-cons bounds for comp-sum losses.
}

\begin{figure}[t]
\begin{center}
\hspace{-.25cm}
\includegraphics[scale=0.35]{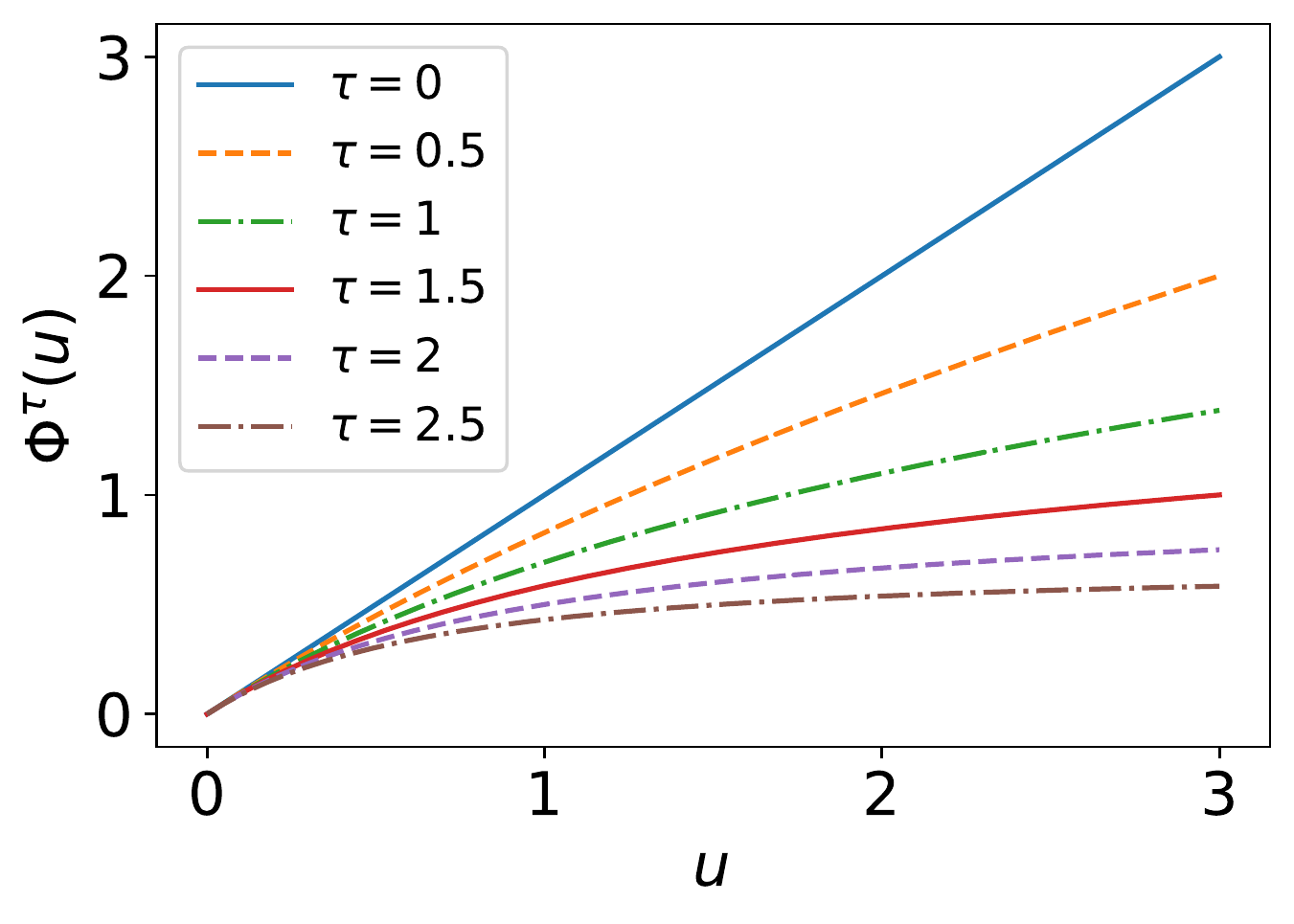}
\vskip -0.15in
\caption{Function $\Phi^{\tau}$ with different values of $\tau$.}
\label{fig:phi}
\end{center}
\vskip -0.2in
\end{figure}

\textbf{Comp-sum losses}.
In this paper, we derive guarantees for \emph{comp-sum losses}, a
family of functions including the logistic loss
that is defined via a composition of two functions $\Phi_1$ and
$\Phi_2$\ignore{, as in \citep{KuznetsovMohriSyed2014}}:
\begin{align}
\label{eq:comp-sum_loss}
\ell_{\Phi_1[\Phi_2]}^{\mathrm{comp}}(h, x, y)
= \Phi_1\paren[\bigg]{\sum_{y'\neq y}\Phi_2\paren*{h(x, y) - h(x, y')}},
\end{align}
where $\Phi_2$ is a non-increasing function upper bounding
$\mathds{1}_{u\leq 0}$ over $u\in\Rset$ and $\Phi_1$ a non-decreasing
auxiliary function.  We will specifically consider $\Phi_2(u) =
\exp(-u)$ as with the loss function related to AdaBoost
\citep{freund1997decision} and $\Phi_1$ chosen out of the following
family of functions $\Phi^\tau$, $\tau \geq 0$, defined for all $u \geq
0$ by
\begin{equation}
\label{eq:Phi1}
\Phi^{\tau}(u) =
\begin{cases}
\frac{1}{1 - \tau} \paren*{(1 + u)^{1 - \tau} - 1} & \tau \geq 0, \tau \neq 1 \\
\log(1 + u) & \tau = 1.
\end{cases}
\end{equation}
Figure~\ref{fig:phi} shows the plot of function $\Phi^{\tau}$ for
different values of $\tau$.  Functions $\Phi^\tau$ verify the
following identities:
\begin{align}
\label{eq:Phi1-derivative}
\frac{\partial \Phi^{\tau}}{\partial u}(u)
= \frac{1}{(1 + u)^{\tau}}, \quad \Phi^{\tau}(0) = 0.
\end{align}
In view of that, by l’H\^opital's rule, $\Phi^\tau$ is continuous as a
function of $\tau$ at $\tau = 1$. To simplify the notation, we will
use $\ell_{\tau}^{\rm{comp}}$ as a short-hand for
$\ell_{\Phi_1[\Phi_2]}^{\mathrm{comp}}$ when $\Phi_1 = \Phi^{\tau}$
and $\Phi_2(u) = \exp(-u)$. $\ell_{\tau}^{\rm{comp}}(h, x, y)$ can be
expressed as follows for any $h$, $x$, $y$ and $\tau \geq 0$:
\begin{multline}
\label{eq:comp-loss}
\ell_{\tau}^{\rm{comp}}(h, x, y)
= \Phi^{\tau}\paren*{\sum_{y'\in \sY} e^{h(x, y') - h(x, y)}-1} \\
=
\begin{cases}
  \frac{1}{1 - \tau}
  \paren*{\bracket*{\sum_{y'\in\sY} e^{{h(x, y') - h(x, y)}}}^{1 - \tau} - 1}
  & \tau \neq 1 \\
\log\paren*{\sum_{y'\in \sY} e^{h(x, y') - h(x, y)}} & \tau = 1.
\end{cases}
\end{multline}
When $\tau = 0$, $\ell_{\tau}^{\rm{comp}}$ coincides with the
sum-exponential loss
\citep{weston1998multi,AwasthiMaoMohriZhong2022multi}
\begin{align*}
  \ell_{\tau = 0}^{\mathrm{comp}}(h, x, y)
  = \sum_{y'\neq y} e^{h(x, y') - h(x, y)}.
\end{align*}
When $\tau = 1$, it coincides with the (multinomial) logistic loss
\citep{Verhulst1838,Verhulst1845,Berkson1944,Berkson1951}:
\begin{align*}
  \ell_{\tau = 1}^{\mathrm{comp}}(h, x, y)
  =- \log \bracket*{\frac{e^{h(x,y)}}{\sum_{y' \in \sY} e^{h(x,y')}}}.
\end{align*}
For $1 < \tau < 2$, it matches the
\emph{generalized cross entropy loss} \citep{zhang2018generalized}:
\begin{align*}
  \ell_{1 < \tau < 2}^{\mathrm{comp}}(h, x, y)
  = \frac{1}{\tau - 1}\bracket*{1 - \bracket*{\frac{e^{h(x,y)}}
    {\sum_{y'\in \sY} e^{h(x,y')}}}^{\tau - 1}},
\end{align*}
for $\tau = 2$, the \emph{mean absolute error loss}
\citep{ghosh2017robust}:
\begin{align*}
  \ell_{\tau = 2}^{\mathrm{comp}}(h, x, y) =
  1 - \frac{e^{h(x,y)}}{\sum_{y'\in \sY} e^{h(x, y')}}.
\end{align*}
Since for any $\tau\geq0$, $\frac{\partial \Phi^{\tau}}{\partial u}$
is non-increasing and satisfies $\frac{\partial \Phi^{\tau}}{\partial
  u}(0)=1$, $\Phi^{\tau}(0) = 0$, for any $\tau \geq 0$,
$\Phi^{\tau}$ is concave, non-decreasing, differentiable,
$1$-Lipschitz, and satisfies that
\begin{equation}
\label{eq:Phi1-property}
\forall u \geq 0,~\Phi^{\tau} (u) \leq u.
\end{equation}

\section{$\sH$-Consistency Bounds for Comp-Sum Losses}
\label{sec:H-consistency-bounds}

In this section, we present and discuss $\sH$-consistency bounds for
comp-sum losses in the standard multi-class classification
scenario. We say that a hypothesis set is \emph{symmetric} when it
does not depend on a specific ordering of the classes, that is, when
there exists a family $\sF$ of functions $f$ mapping from $\sX$ to
$\Rset$ such that $\curl*{\bracket*{h(x, 1),\ldots,h(x, c)}\colon h\in
  \sH} = \curl*{\bracket*{f_1(x),\ldots, f_c(x)}\colon f_1, \ldots,
  f_c\in \sF}$, for any $x \in \sX$. We say that a hypothesis set
$\sH$ is \emph{complete} if the set of scores it generates spans
$\Rset$, that is, $\curl*{h(x, y)\colon h\in \sH} = \Rset$, for any
$(x, y)\in \sX \times \sY$. The hypothesis sets widely used in
practice are all symmetric and complete.

\subsection{$\sH$-Consistency Guarantees}

The following holds for all comp-sum loss functions and
all symmetric and complete hypothesis sets, which includes
those typically considered in applications.

\begin{restatable}[\textbf{$\sH$-consistency bounds for comp-sum losses}]
  {theorem}{BoundCompSum}
\label{Thm:bound_comp_sum}
Assume that $\sH$ is symmetric and complete. Then, for any $\tau\in
[0,\infty)$ and any $h \in \sH$, the following inequality holds:
\ifdim\columnwidth= \textwidth
\begin{align*}
\sR_{\ell_{0-1}}(h)-\sR_{\ell_{0-1}}^*(\sH)
\leq \Gamma_{\tau}
  \paren*{\sR_{\ell_{\tau}^{\rm{comp}}}(h) - \sR_{\ell_{\tau}^{\rm{comp}}}^*(\sH)
    + \sM_{\ell_{\tau}^{\rm{comp}}}(\sH)}
- \sM_{\ell_{0-1}}(\sH),
\end{align*}
\else
\begin{align*}
& \sR_{\ell_{0-1}}(h)-\sR_{\ell_{0-1}}^*(\sH)\\
& \mspace{-6mu} \leq \Gamma_{\tau}
  \paren*{\sR_{\ell_{\tau}^{\rm{comp}}}(h) - \sR_{\ell_{\tau}^{\rm{comp}}}^*(\sH)
    + \sM_{\ell_{\tau}^{\rm{comp}}}(\sH)}
\!-\! \sM_{\ell_{0-1}}(\sH),
\end{align*}
\fi where $\Gamma_{\tau}(t)= \sT_{\tau}^{-1}(t)$ is the inverse of
$\sH$-consistency comp-sum transformation, defined for all $\beta \in
    [0,1]$ by
$\sT_{\tau}(\beta)
= \begin{cases}
\frac{2^{1-\tau}}{1-\tau}\bracket*{1 -\bracket*{\frac{\paren*{1 + \beta}^{\frac1{2 - \tau }} +  \paren*{1 - \beta}^{\frac1{2 - \tau }}}{2}}^{2 - \tau }} & \tau \in [0,1)\\
\frac{1+\beta}{2}\log\bracket*{1+\beta} + \frac{1-\beta}{2}\log\bracket*{1-\beta} & \tau =1 \\
\frac{1}{(\tau-1)n^{\tau-1}}\bracket*{\bracket*{\frac{\paren*{1 + \beta}^{\frac1{2 - \tau }} +  \paren*{1 - \beta}^{\frac1{2 - \tau }}}{2}}^{2 - \tau } \mspace{-20mu} -1} & \tau \in (1,2)\\
\frac{1}{(\tau-1)n^{\tau-1}}\,
\beta & \tau \in [2,\plus \infty).
\end{cases}$ 
\end{restatable}
By l’H\^opital's rule, $\sT_{\tau}$ is continuous as a function of
$\tau$ at $\tau = 1$. Using the fact that $\lim_{x\to
  0^{+}}\paren*{a^{\frac1x}+b^{\frac1x}}^x= \max\curl*{a,b}$,
$\sT_{\tau}$ is continuous as a function of $\tau$ at $\tau =
2$. Furthermore, for any $\tau\in [0,\plus\infty)$, $\sT_{\tau}$ is a
  convex and increasing function, and satisfies that
  $\sT_{\tau}(0)=0$. Note that for the sum-exponential loss ($\tau=0$)
  and logistic loss ($\tau=1$), the expression $\sT_{\tau}$ matches
  that of their binary $\sH$-consistency estimation error
  transformation $1-\sqrt{1-t^2}$ and
  $\frac{1+t}{2}\log(1+t)+\frac{1-t}{2}\log(1-t)$ in the binary
  classification setting \cite{awasthi2022Hconsistency}, which were
  proven to be tight. We will show that, for these loss functions and
  in this multi-class classification setting, $\sT_{\tau}$s admit a
  tight functional forms as well. We illustrate the function
  $\Gamma_{\tau}$ with different values of $\tau$ in
  Figure~\ref{fig:Gamma}.

\begin{figure}[t]
\begin{center}
\hspace{-.25cm}
\includegraphics[scale=0.35]{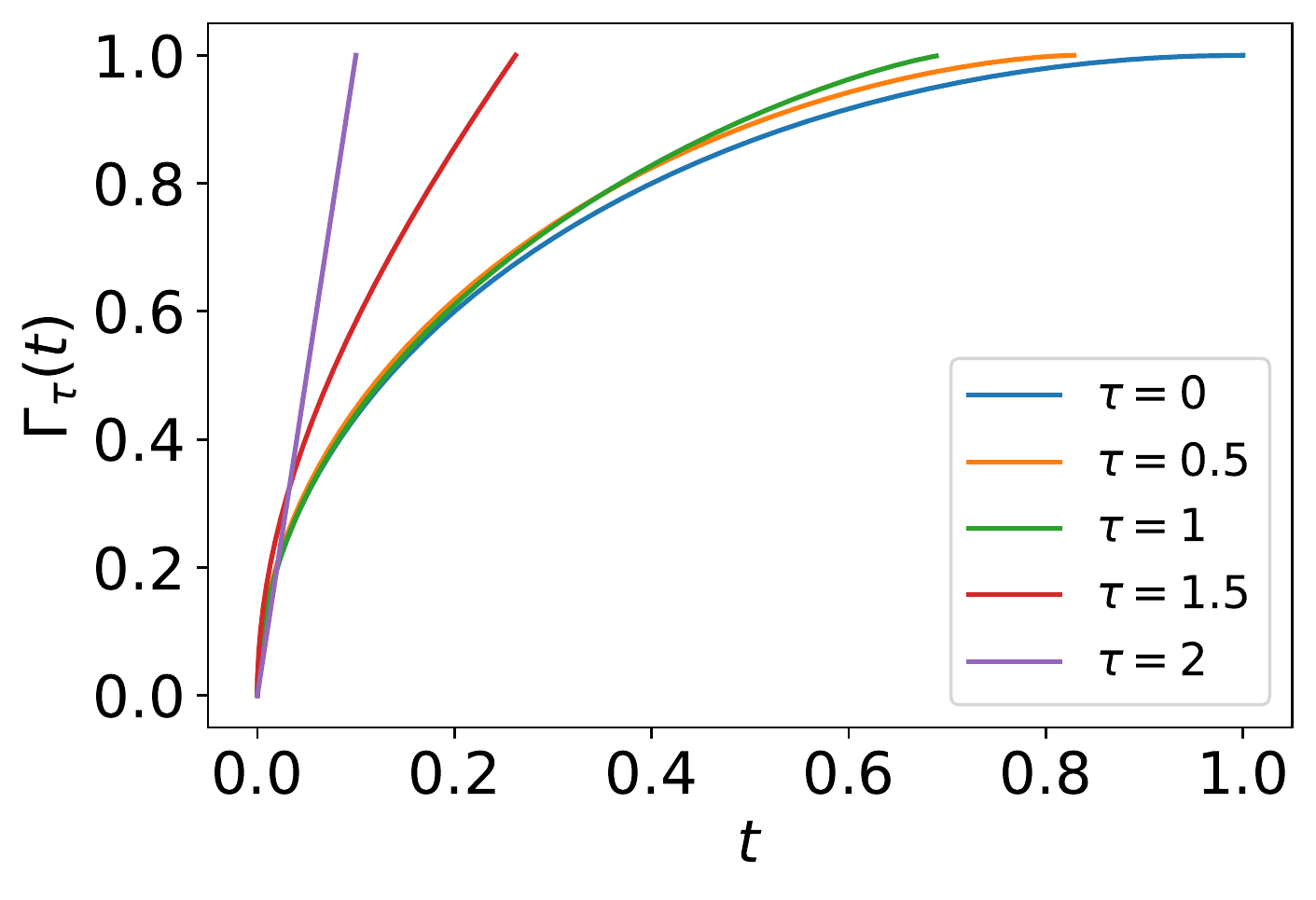}
\vskip -0.15in
\caption{Function $\Gamma_{\tau}$ with different values of $\tau$ for $n = 10$.}
\label{fig:Gamma}
\end{center}
\vskip -0.2in
\end{figure}

\begin{figure*}[t]
\begin{center}
\hspace{-.25cm}
\includegraphics[scale=0.3]{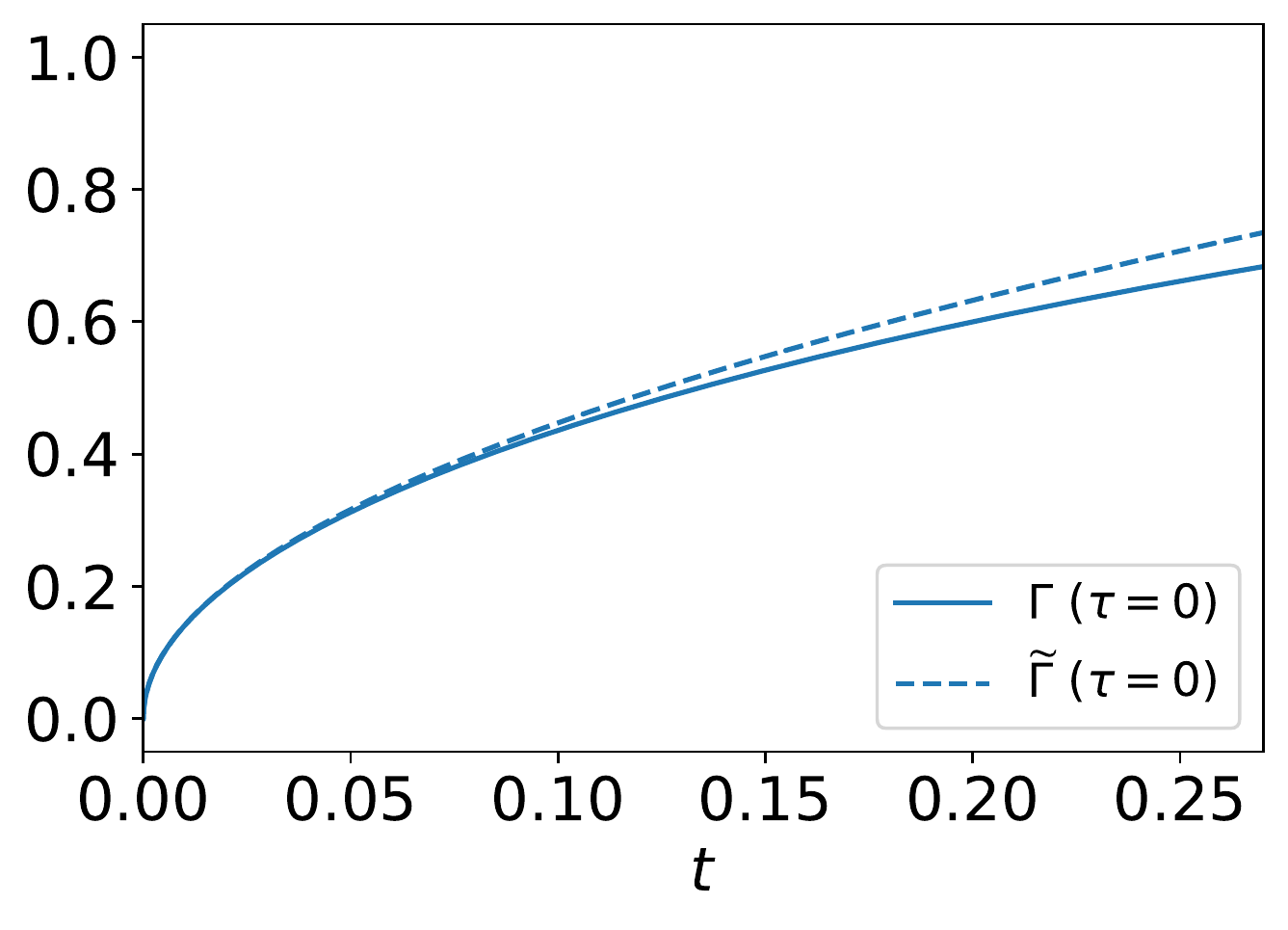}
\includegraphics[scale=0.3]{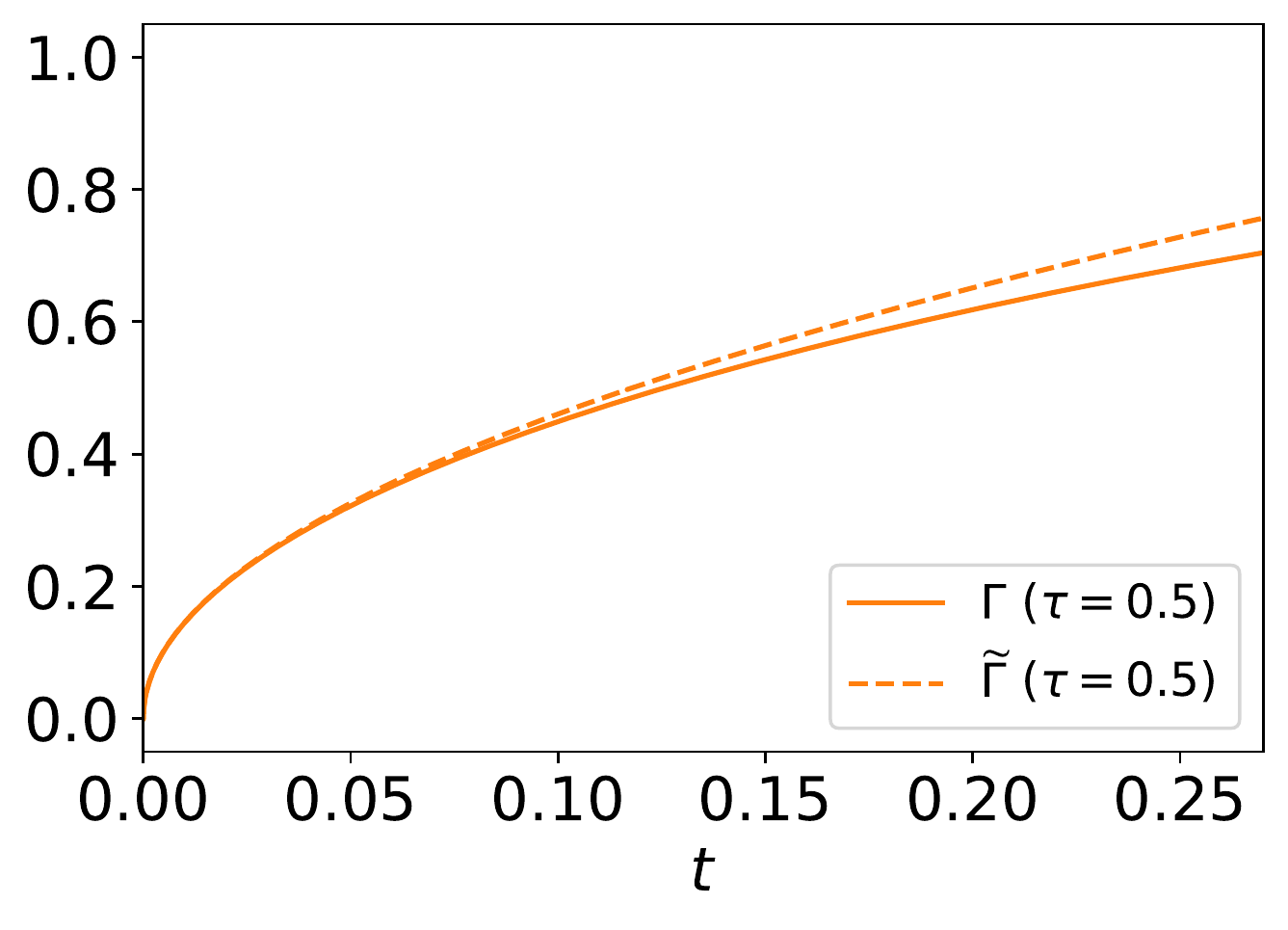}
\includegraphics[scale=0.3]{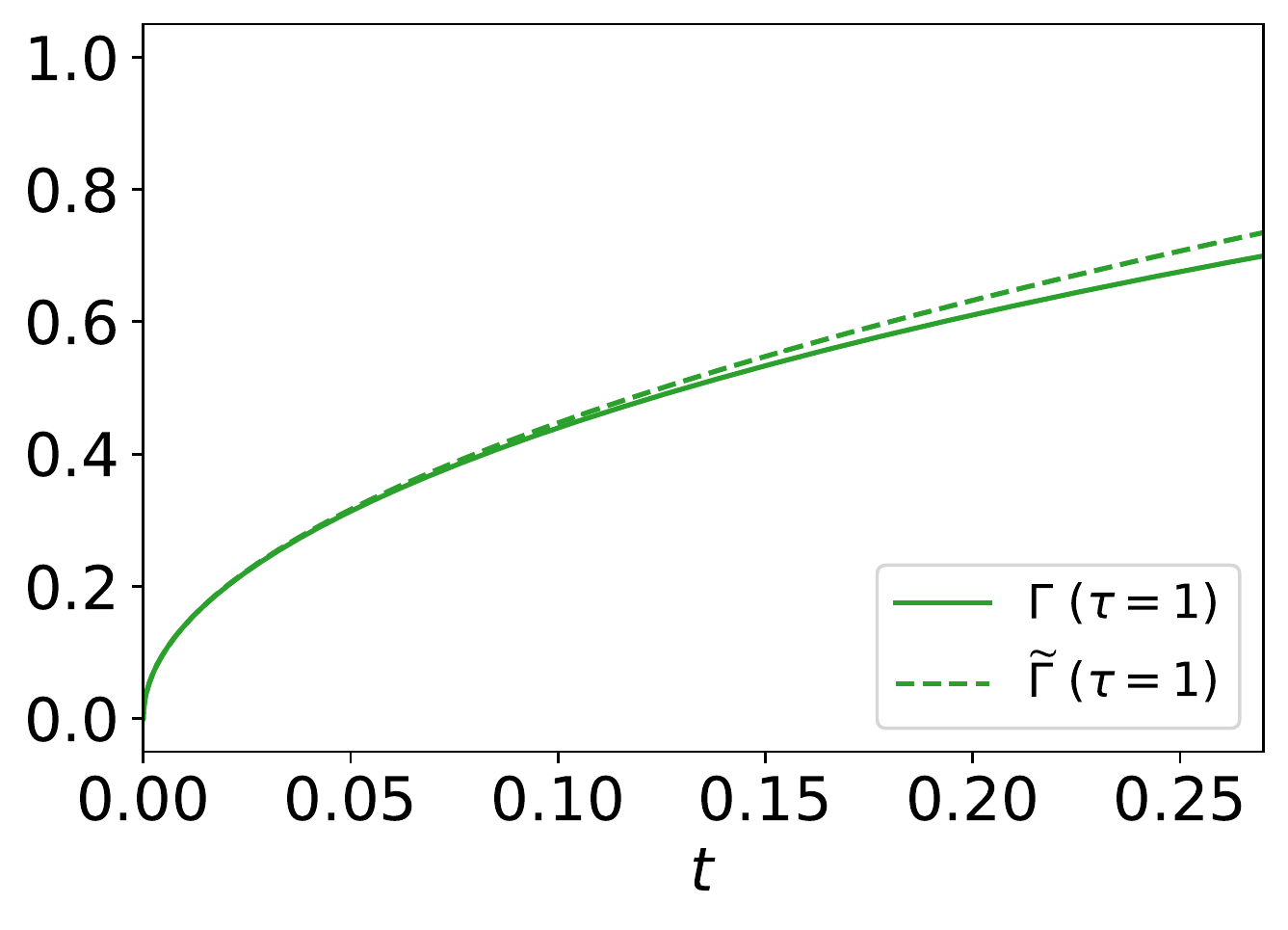}
\includegraphics[scale=0.3]{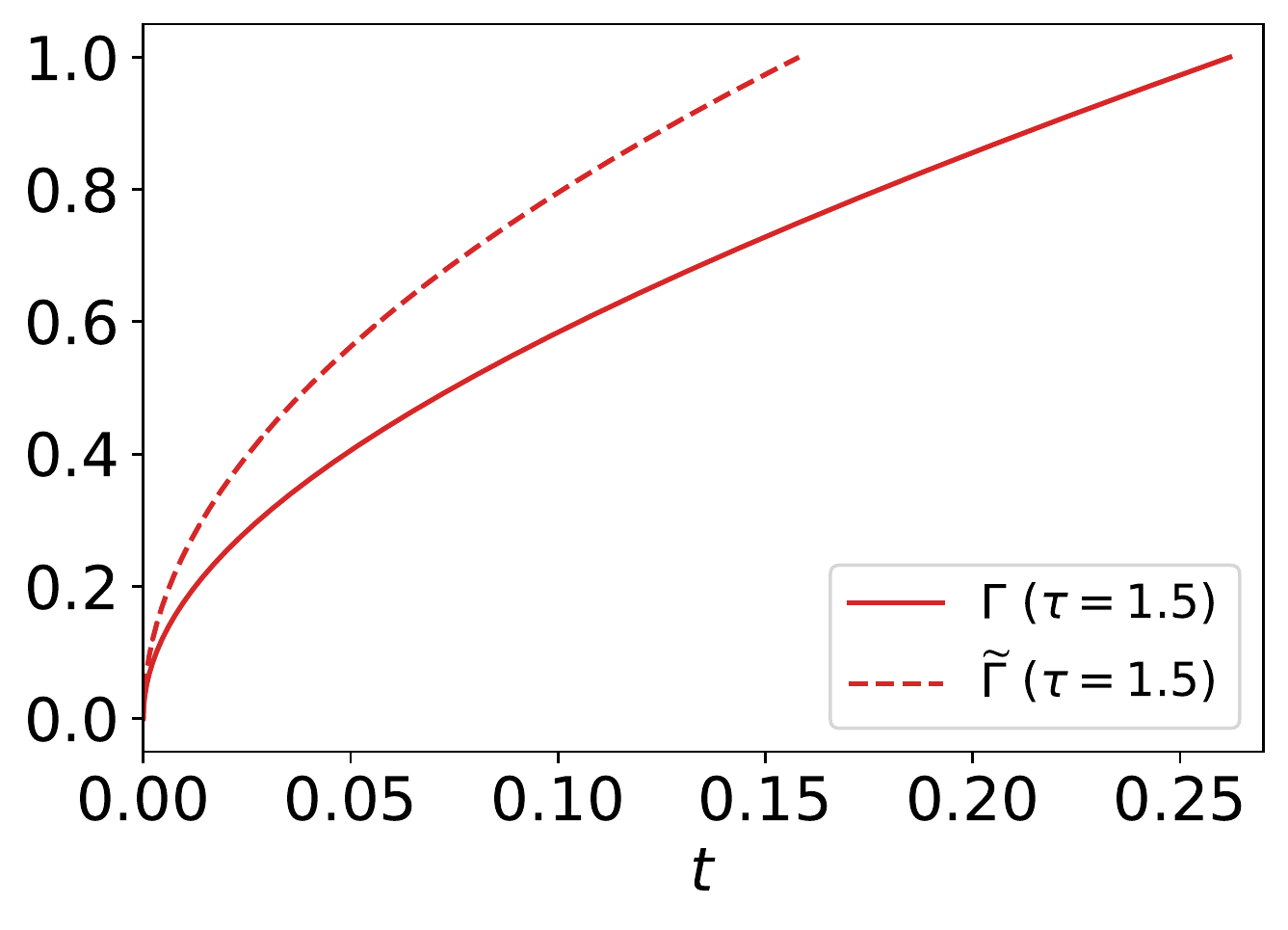}
\vskip -0.15in
\caption{Function $\Gamma_{\tau}$ and its upper bound $\wt\Gamma_{\tau}$ with different values of $\tau$ and $n = 10$.}
\label{fig:Gamma-wt}
\end{center}
\vskip -0.2in
\end{figure*}

By using Taylor expansion, $\sT_{\tau}(\beta)$ can be lower bounded by
its polynomial approximation with the tightest order as
\begin{equation}
\label{eq:wt-sT}
\sT_{\tau}(\beta)\geq \wt\sT_{\tau}(\beta)=
\begin{cases}
\frac{\beta^2}{2^{\tau}(2-\tau)} & \tau \in [0,1)\\
\frac{\beta^2}{2n^{\tau-1}} & \tau \in [1,2)\\
\frac{\beta}{(\tau-1)n^{\tau-1}}& \tau \in [2,\plus \infty).
\end{cases}
\end{equation}
Accordingly, $\Gamma_{\tau}(t)$ can be upper bounded by the inverse of
$\wt\sT_{\tau}$, which is denoted by $\wt \Gamma_{\tau}(t)=
\wt\sT_{\tau}^{-1}(t)$, as shown below
\begin{equation}
\label{eq:wt-Gamma}
\Gamma_{\tau}(t)\leq  \wt \Gamma_{\tau}(t)= \begin{cases}
\sqrt{2^{\tau}(2-\tau) t} & \tau\in [0,1)\\
\sqrt{2n^{\tau-1} t } & \tau\in [1,2) \\
(\tau - 1) n^{\tau - 1} t & \tau \in [2,\plus \infty).
\end{cases}
\end{equation}
A detailed derivation is given in
Appendix~\ref{app:Gamma-upper-bound}. The plots of function
$\Gamma_{\tau}$ and their corresponding upper bound $\wt\Gamma_{\tau}$
($n=10$) are shown in Figure~\ref{fig:Gamma-wt}, for different values
of $\tau$; they illustrate the quality of the approximations via
$\wt\Gamma_{\tau}$.

Recall that the minimizability gaps vanish when $\sH$ is the family of
all measurable functions or when $\sH$ contains the Bayes predictor.
In their absence, the theorem
shows that if the estimation loss
$(\sR_{\ell_{\tau}^{\rm{comp}}}(h) -
\sR_{\ell_{\tau}^{\rm{comp}}}^*(\sH))$ is reduced to $\e$, then, for
$\tau \in [0, 2)$, in particular for the logistic loss ($\tau = 1$)
  and the generalized cross-entropy loss ($\tau \in (1, 2)$), modulo a
  multiplicative constant, the zero-one estimation loss
  $(\sR_{\ell_{0-1}}(h) - \sR_{\ell_{0-1}}^*(\sH))$ is bounded by
  $\sqrt{\e}$. For the logistic loss, the following 
  guarantee holds for all $h \in \sH$:
\[
\sR_{\ell_{0-1}}(h) - \sR_{\ell_{0-1}}^*(\sH)
\leq 
  \sqrt{2 \paren[\big]{\sR_{\ell_{1}^{\rm{comp}}}(h) - \sR_{\ell_{1}^{\rm{comp}}}^*(\sH)}}.
\]
The bound is even more favorable for the mean absolute error loss
($\tau = 2$) or for comp-sum losses $\ell_{\tau}^{\rm{comp}}$ with
$\tau \in (2, +\infty)$ since in that case, modulo a multiplicative
constant, the zero-one estimation loss $(\sR_{\ell_{0-1}}(h) -
\sR_{\ell_{0-1}}^*(\sH))$ is bounded by $\e$. In general, the
minimizability gaps are not null however and, in addition to the
functional form of $\Gamma_\tau$, two other key features help compare
comp-sum losses: (i) the magnitude of the minimizability gap
$\sM_{\ell_{\tau}^{\rm{comp}}}(\sH)$; and (ii) the dependency of the
multiplicative constant on the number of classes, which makes it less
favorable for $\tau \in (1, +\infty)$. Thus, we will specifically
further analyze the minimizability gaps in the next section
(Section~\ref{sec:comparison}).

The proof of the theorem is given in
Appendix~\ref{app:bound_comp-sum}. It consists of using the general
$\sH$-consistency bound tools given by \citet{awasthi2022Hconsistency,
  AwasthiMaoMohriZhong2022multi} and of analyzing the calibration gap
of the loss function $\ell_{\tau}^{\rm{comp}}$ for different values of
$\tau$ in order to lower bound it in terms of the zero-one loss
calibration gap. As pointed out by
\citet{AwasthiMaoMohriZhong2022multi}, deriving such bounds is
non-trivial in the multi-class classification setting. In the proof,
we specifically choose auxiliary functions $\ov h_{\mu}$ target to the
comp-sum losses, which satisfies the property $\sum_{y\in \sY}e^{h(x,
  y)} = \sum_{y\in \sY}e^{\ov h_{\mu}(x, y)}$. Using this property, we
then establish several general lemmas that are applicable to any
$\tau\in [0,\infty)$ and are helpful to lower bound the calibration
  gap of $\ell_{\tau}^{\rm{comp}}$. This is significantly different
  from the proofs of \citet{AwasthiMaoMohriZhong2022multi}
  whose analysis depends on concrete loss functions case by
  case. Furthermore, our proof technique actually leads to the
  tightest bounds as shown below. Our proofs are novel and cover the
  full comp-sum loss family, which includes the logistic loss. Next,
  we further prove that the functional form of our bounds
  $\sH$-consistency bounds cannot be improved.

\begin{restatable}[\textbf{Tightness}]{theorem}{Tightness}
\label{Thm:tightness}
Assume that $\sH$ is symmetric and complete. Then, for any $\tau\in [0,1]$ and $\beta \in[0,1]$, there exist a distribution $\sD$ and a
hypothesis $h\in\sH$ such that $\sR_{\ell_{0-1}}(h)-
\sR_{\ell_{0-1},\sH}^*+\sM_{\ell_{0-1},\sH}= \beta$ and
$\sR_{\ell_{\tau}^{\rm{comp}}}(h) - \sR_{\ell_{\tau}^{\rm{comp}}}^*(\sH) + \sM_{\ell_{\tau}^{\rm{comp}}}(\sH)=
\sT_{\tau}(\beta)$.
\end{restatable}
The proof is given in Appendix~\ref{app:bound_comp-sum}.  The theorem
shows that the bounds given by the $\sH$-consistency comp-sum
transformation $\sT_{\tau}$, or, equivalently, by its inverse
$\Gamma_{\tau}$ in Theorem~\ref{Thm:bound_comp_sum} is tight for any
$\tau\in [0, 1]$, which includes as special cases the logistic loss
($\tau = 1$).

\subsection{Learning Bounds}

Our $\sH$-consistency bounds can be used to derive zero-one learning
bounds for a hypothesis set $\sH$. For a sample size $m$, let
$\Rad_m^\tau(\sH)$ denote the Rademacher complexity of the family of
functions $\curl*{(x, y) \mapsto \ell_{\tau}^{\rm{comp}}(h, x, y)
  \colon h \in \sH}$ and $B_\tau$
an upper bound on the loss $\ell_{\tau}^{\rm{comp}}$.
\ignore{
$\ell_{\tau}^{\rm{comp}}(h, x, y)$, for $h \in \sH$ and $(x, y) \times
  \sX \times \sY$.
}
\begin{restatable}{theorem}{GenBound}
\label{th:genbound}
With probability at least $1 - \delta$ over the draw of a sample $S$
from $\sD^m$, the following zero-one loss estimation bound holds
for an empirical minimizer $\h h_S \in \sH$ of the comp-sum loss
$\ell_{\tau}^{\rm{comp}}$ over $S$:
\ifdim\columnwidth= \textwidth
{
\begin{align*}
\sR_{\ell_{0-1}}(\h h_S) - \sR_{\ell_{0-1}}^*(\sH)
\leq \Gamma_{\tau}
  \paren[\bigg]{\sM_{\ell_{\tau}^{\rm{comp}}}(\sH) + 4 \Rad_m^\tau(\sH) +
2 B_\tau \sqrt{\tfrac{\log \frac{2}{\delta}}{2m}}}
- \sM_{\ell_{0-1}}(\sH).
\end{align*}
}
\else
{
\begin{align*}
& \sR_{\ell_{0-1}}(\h h_S) - \sR_{\ell_{0-1}}^*(\sH)\\
& \mspace{-6mu} \leq \Gamma_{\tau}
  \paren[\bigg]{\sM_{\ell_{\tau}^{\rm{comp}}}(\sH) + 4 \Rad_m^\tau(\sH) +
2 B_\tau \sqrt{\tfrac{\log \frac{2}{\delta}}{2m}}}
\!-\! \sM_{\ell_{0-1}}(\sH).
\end{align*}
}
\fi
\end{restatable}
The proof is given in Appendix~\ref{app:genbound}. To our knowledge,
these are the first zero-one estimation loss guarantees for empirical
minimizers of a comp-sum loss such as the logistic loss. Our previous
comments about the properties of $\Gamma_\tau$, in particular its
functional form or its dependency on the number of classes $n$,
similarly apply here. These are precise bounds that take into account
the minimizability gaps.

\section{Comparison of Minimizability Gaps}
\label{sec:comparison}

We now further analyze these quantities and make our guarantees even
more explicit.
Consider a composed loss function defined by $(\Phi_1 \circ \ell_2)(h,
x, y)$, for all $h \in \sH$ and $(x, y) \in \sX \times \sY$, with
$\Phi_1$ concave and non-decreasing.
Then, by Jensen's inequality, we can write:
\begin{equation}
\begin{aligned}
\label{eq:concave-Phi1}
  \sR^*_{\Phi_1 \circ \ell_2}(\sH)
  & = \inf_{h \in \sH} \curl*{\E_{(x, y) \sim \sD}[(\Phi_1 \circ \ell_2)(h, x, y)] }\\
  & \leq \inf_{h \in \sH} \curl*{\Phi_1\paren*{\E_{(x, y) \sim \sD}[\ell_2(h, x, y)]}}\\
  & = \Phi_1\paren*{\inf_{h \in \sH} \curl*{\E_{(x, y) \sim \sD}[\ell_2(h, x, y)]}}\\
  & = \Phi_1\paren*{\sR^*_{\ell_2}(\sH)}.
\end{aligned}
\end{equation}
Recall that the comp-sum losses $\ell_{\tau}^{\rm{comp}}$ can be
written as $\ell_{\tau}^{\rm{comp}}= \Phi^{\tau}\circ
\ell_{\tau=0}^{\rm{comp}}$, where $\ell_{\tau=0}^{\rm{comp}}(h, x, y)
= \sum_{y'\neq y}\exp\paren*{h(x, y') - h(x, y)}$. $\Phi^\tau$ is
concave since we have $\frac{\partial^2 \Phi^\tau}{\partial^2 u}(u) =
\frac{-\tau}{(1 + u)^{\tau + 1}} \leq 0$ for all $\tau \geq 0$ and $u
\geq 0$. Using these observations, the following results can be shown.
\begin{restatable}[\textbf{Characterization of minimizability gaps - stochastic case}]
  {theorem}{GapUpperBound}
\label{Thm:gap-upper-bound}
Assume that $\sH$ is symmetric and complete. Then, for the comp-sum losses $\ell_{\tau}^{\rm{comp}}$, the minimizability gaps can be upper bounded as follows:
\begin{align}
\label{eq:gap-upper-bound}
\sM_{\ell_{\tau}^{\rm{comp}}}(\sH)
\leq \Phi^{\tau}\paren*{\sR^*_{\ell_{\tau=0}^{\rm{comp}}}(\sH)} - \E_x[\sC^*_{\ell_{\tau}^{\rm{comp}}}(\sH, x)],
\end{align}
where $\sC^*_{\ell_{\tau}^{\rm{comp}}}(\sH, x)$ is given by 
\begin{equation}
\label{eq:comp-sum-Cstar}
\mspace{-5mu}
\begin{cases}
\frac{1}{1 - \tau} \paren*{\bracket*{\sum_{y\in \sY}p(x,y)^{\frac{1}{2-\tau}}}^{2 - \tau} - 1} & \tau\geq 0, \tau\neq1, \tau \neq 2\\
-\sum_{y\in \sY} p(x,y) \log\bracket*{p(x,y) } & \tau=1\\
1 - \max_{y\in \sY}p(x,y) & \tau =2.
\end{cases}
\mspace{-20mu}
\end{equation}
\ignore{is a decreasing function of $\tau \geq 0$ for any distribution.}
\end{restatable}
Note that the expressions for $\sC^*_{\ell_{\tau}^{\rm{comp}}}(\sH, x)$ in
\eqref{eq:comp-sum-Cstar} can be formulated in terms of the $(2 -
\tau)$-R\'enyi entropy.

\begin{restatable}[\textbf{Characterization of minimizability gaps - deterministic case}]
  {theorem}{GapUpperBoundDetermi}
\label{Thm:gap-upper-bound-determi}
Assume that for any $x \in \sX$, we have $\curl*{\paren*{h(x, 1),
    \ldots, h(x, n)}\colon h \in \sH}$ = $[-\Lambda,
  +\Lambda]^n$. Then, for comp-sum losses $\ell_{\tau}^{\rm{comp}}$
and any deterministic distribution, the minimizability gaps can be
upper bounded as follows:
\begin{align}
\label{eq:gap-upper-bound-determi}
\sM_{\ell_{\tau}^{\rm{comp}}}(\sH)
\leq \Phi^{\tau}\paren*{\sR^*_{\ell_{\tau=0}^{\rm{comp}}}(\sH)} - \sC^*_{\ell_{\tau}^{\rm{comp}}}(\sH, x),
\end{align}
where $\sC^*_{\ell_{\tau}^{\rm{comp}}}(\sH, x)$ is given by 
\begin{equation}
\label{eq:comp-sum-Cstar-determi}
\begin{cases}
\frac{1}{1 - \tau} \paren*{\bracket*{1 + e^{-2 \Lambda}(n - 1)}^{1 - \tau} - 1} & \tau\geq 0, \tau\neq1\\
\log\bracket*{1 + e^{-2 \Lambda}(n - 1) } & \tau=1.
\end{cases}
\end{equation}
\end{restatable}
The proofs of these theorems are given in
Appendix~\ref{app:gap-upper-bound}. Note that, when $\tau = 0$,
$\Phi^{\tau}(u)=u$ gives the sum exponential loss $\Phi^{\tau}\circ
\ell_{\tau=0}^{\rm{comp}} = \ell_{\tau=0}^{\rm{comp}}$. For
deterministic distributions, by \eqref{eq:comp-sum-Cstar-determi}, we
obtain $\sC^*_{\ell_{\tau=0}^{\rm{comp}}}(\sH, x) = e^{-2 \Lambda}(n -
1)$.  Therefore, \eqref{eq:comp-sum-Cstar-determi} can be rewritten as
$\sC^*_{\ell_{\tau}^{\rm{comp}}}(\sH, x) =
\Phi^{\tau}\paren*{\sC^*_{\ell_{\tau=0}^{\rm{comp}}}(\sH, x)}$.  Thus,
inequality~\eqref{eq:gap-upper-bound-determi} can be rewritten as
follows:
\begin{align}
\label{eq:gap-upper-bound-determi-final}
\mspace{-10mu}
\sM_{\ell_{\tau}^{\rm{comp}}}(\sH)
&\leq \Phi^{\tau}\paren*{\sR^*_{\ell_{\tau=0}^{\rm{comp}}}(\sH)} - \Phi^{\tau}\paren*{\sC^*_{\ell_{\tau=0}^{\rm{comp}}}(\sH, x)}.
\end{align}
We will denote the right-hand side by $\wt \sM_{\ell_{\tau}^{\rm{comp}}}(\sH)$, 
$\wt \sM_{\ell_{\tau}^{\rm{comp}}}(\sH) = \Phi^{\tau}\paren*{\sR^*_{\ell_{\tau=0}^{\rm{comp}}}(\sH)} - \Phi^{\tau}\paren*{\sC^*_{\ell_{\tau=0}^{\rm{comp}}}(\sH, x)}$.
Note that we always have $\sR^*_{\ell_{\tau=0}^{\rm{comp}}}(\sH)\geq
\E_x\bracket*{\sC^*_{\ell_{\tau=0}^{\rm{comp}}}(\sH, x)}$. Here,
$\E_x\bracket[big]{\sC^*_{\ell_{\tau=0}^{\rm{comp}}}(\sH,
  x)}= \sC^*_{\ell_{\tau=0}^{\rm{comp}}}(\sH, x)$ since
$\sC^*_{\ell_{\tau=0}^{\rm{comp}}}(\sH, x)$ is independent of $x$ as shown in \eqref{eq:comp-sum-Cstar-determi}. Then,
\eqref{eq:gap-upper-bound-determi-final} can be used to compare the
minimizability gaps for different $\tau$. 

\begin{restatable}{lemma}{LemmaCompare}
\label{lemma:lemma-compare}
For any $u_1 \geq u_2\geq 0$, $\Phi^{\tau}(u_1)-\Phi^{\tau}(u_2)$ is non-increasing with respect to $\tau$.
\end{restatable}
The proof is given in
Appendix~\ref{app:lemma-compare}.
Lemma~\ref{lemma:lemma-compare} implies that $\wt \sM_{\ell_{\tau}^{\rm{comp}}}(\sH)$ is a non-increasing function of $\tau$. Thus, given a hypothesis set $\sH$, we have:
\begin{equation}
\label{eq:order-M}
\mspace{-3mu} \wt \sM_{\ell_{\tau=0}}(\sH) \geq \wt \sM_{\ell_{\tau=1}}(\sH) \geq  \wt\sM_{\ell_{1<\tau<2}}(\sH)\geq \wt \sM_{\ell_{\tau=2}}(\sH). \mspace{-3mu}
\end{equation}
By Section~\ref{sec:preliminaries}, these minimizability gaps
specifically correspond to that of sum-exponential loss ($\tau=0$),
logistic loss ($\tau=1$), generalized cross-entropy loss ($1<\tau<2$)
and mean absolute error loss ($\tau=2$) respectively. Note that for
those loss functions, by Theorem~\ref{Thm:bound_comp_sum}, when the
estimation error $\sR_{\ell_{\tau}^{\rm{comp}}}(h) -
\sR_{\ell_{\tau}^{\rm{comp}}}^*(\sH)$ is minimized to zero, the
estimation error of zero-one classification loss is upper bounded by
$\wt \Gamma_{\tau}\paren*{\sM_{\ell_{\tau}}}$. Therefore,
\eqref{eq:order-M} combined with the form of $\wt\Gamma_{\tau}$ helps
compare the sum-exponential loss ($\tau=0$), logistic loss ($\tau=0$),
generalized cross-entropy loss ($1<\tau<2$) and mean absolute error
loss ($\tau=2$) in practice. See Section~\ref{sec:experiments-non-adv}
for a discussion of the empirical results in light of these
theoretical findings.

\section{Smooth Adversarial Comp-Sum Losses}
\label{sec:adversarial}

A recent challenge in the application of neural networks is their
robustness to imperceptible perturbations
\citep{szegedy2013intriguing}. While neural networks trained on large
datasets have achieved breakthroughs in speech and visual recognition
tasks in recent years
\citep{SutskeverVinyalsLe2014,KrizhevskySutskeverHinton2012}, their
accuracy remains substantially lower in the presence of such
perturbations even for state-of-the-art robust algorithms. One key
factor in the design of robust algorithms is the choice of the
surrogate loss function used for training since directly optimizing
the target adversarial zero-one loss with most hypothesis sets is
NP-hard. To tackle this problem, we introduce a family of loss
functions designed for adversarial robustness that we call
\emph{smooth adversarial comp-sum loss functions}. These are loss
functions obtained by augmenting comp-sum losses with a natural
corresponding smooth term.  We show that these loss functions are
beneficial in the adversarial setting by proving that they admit
$\sH$-consistency bounds. This leads to a family of algorithms for
adversarial robustness that consist of minimizing a regularized smooth
adversarial comp-sum loss.

\subsection{Definition}

In adversarial robustness, the target adversarial zero-one
classification loss is defined as the worst loss incurred over an
$\ell_p$ perturbation ball of $x$ with perturbation size $\gamma$, $p
\in [1, +\infty]$, $\sfB_p(x, \gamma) = \curl*{x'\colon \norm*{x -
    x'}_p \leq \gamma}$:
\[
\ell_{\gamma}(h, x, y) = \sup_{x' \in \sfB_p(x, \gamma)} \ell_{0-1}(h, x', y).
\]
We first introduce the adversarial comp-sum
$\rho$-margin losses, which is defined as the supremum based counterpart of comp-sum losses \eqref{eq:comp-sum_loss} with $\Phi_1 = \Phi^{\tau}$ and $\Phi_2(u) =
\Phi_{\mathrm{\rho}}(u) = \min\curl*{\max\curl*{0, 1 -
    \frac{u}{\rho}}, 1}$, the $\rho$-margin loss function (see for
example \citep{MohriRostamizadehTalwalkar2018}):
\begin{align*}
\wt
\ell^{\mathrm{comp}}_{\tau,\rho}(h, x, y)
\mspace{-5mu}
=
\mspace{-10mu}
\sup_{x':\norm*{x - x'}_p\leq \gamma}
\mspace{-25mu}
\Phi^{\tau}
\paren*{\sum_{y' \neq y}\Phi_{\rho}\paren*{h(x', y') - h(x', y)}}.
\end{align*}
In the next section, we will show that $\wt
\ell^{\mathrm{comp}}_{\tau,\rho}$ admits an $\sH$-consistency bound
with respect to the adversarial zero-one loss $\ell_{\gamma}$. Since
$\Phi_{\rho}$ is not-convex, we will further derive the \emph{smooth
adversarial comp-sum loss} based on $\wt
\ell^{\mathrm{comp}}_{\tau,\rho}$, that has similar $\sH$-consistency
guarantees and is better to optimize.  By the expression of the
derivative of $\Phi^\tau$ in \eqref{eq:Phi1-derivative}, for all $\tau
\geq 0$ and $u \geq 0$, we have $\abs*{\frac{\partial
    \Phi^{\tau}}{\partial u}(u)} = \frac{1}{(1 + u)^{\tau}} \leq 1$,
thus $\Phi^{\tau}$ is $1$-Lipschitz over $\Rset_+$. Define
$\Delta_h(x, y, y') = h(x, y) - h(x, y')$ and let $\ov \Delta_h(x, y)$
denote the $(n - 1)$-dimensional vector $\paren[big]{\Delta_h(x, y,
  1), \ldots, \Delta_h(x, y, y - 1), \Delta_h(x, y, y + 1), \ldots,
  \Delta_h(x, y, n)}$.  For any $\tau \geq 0$, since $\Phi^{\tau}$ is
$1$-Lipschitz and non-decreasing, we have:
\begin{align*}
& \wt \ell^{\mathrm{comp}}_{\tau,\rho}(h, x, y) 
- \ell^{\mathrm{comp}}_{\tau,\rho}(h, x, y)\\
& \sup_{x' \in \sfB(x, \gamma)}\sum_{y'\neq y}\Phi_{\rho}\paren*{-\Delta_h(x',y,y')}
- \ignore{\sum_{y'\neq y}} \Phi_{\rho}\paren*{-\Delta_h(x, y, y')}.
\end{align*}
Since $\Phi_{\rho}(u)$ is $\frac{1}{\rho}$-Lipschitz, 
by the Cauchy-Schwarz inequality, for any
$\nu \geq \frac{\sqrt{n - 1}}{\rho} \geq \frac{1}{\rho}$, we have
\begin{align*}
& \wt \ell^{\mathrm{comp}}_{\tau,\rho}(h, x, y) \\
  & \leq  \ell^{\mathrm{comp}}_{\tau,\rho}(h, x, y) + \nu \sup_{x' \in \sfB(x, \gamma)} \norm*{\ov \Delta_h(x',y)-\ov \Delta_h(x,y)}_2\\
  & \leq  \ell^{\mathrm{comp}}_{\tau}\paren*{\frac{h}{\rho},x,y} + \nu \sup_{x' \in \sfB(x, \gamma)} \norm*{\ov \Delta_h(x',y)-\ov \Delta_h(x,y)}_2,
\end{align*}
where we used the inequality $\exp\paren*{-u/\rho}\geq
\Phi_{\rho}(u)$.  We will refer to a loss function defined by the last
expression as a \emph{smooth adversarial comp-sum loss} and denote it
by $\ell^{\mathrm{comp}}_{\mathrm{smooth}}$. In the next section, we
will provide strong $\sH$-consistency guarantees for
$\ell^{\mathrm{comp}}_{\mathrm{smooth}}$.

\subsection{Adversarial $\sH$-Consistency Guarantees}

To derive guarantees for our smooth adversarial comp-sum loss, we
first prove an adversarial $\sH$-consistency bound for adversarial
comp-sum $\rho$-margin losses $\wt \ell^{\mathrm{comp}}_{\tau,\rho}$
for any symmetric and \emph{locally $\rho$-consistent} hypothesis set.
\begin{definition}
We say that a hypothesis set $\sH$ is \emph{locally $\rho$-consistent}
if for any $x\in \sX$, there exists a hypothesis $h \in \sH$ such that
$\inf_{x'\colon \norm*{x - x'}\leq \gamma}\abs*{h(x', i) - h(x',
  j)}\geq \rho>0$ for any $i\neq j \in \sY$ and for any $x'\in
\curl*{x'\colon \norm*{x - x'}\leq \gamma}$, $\curl*{h(x',y):y\in
  \sY}$ has the same ordering.
\end{definition}
Common hypothesis sets used in practice, such as the family of linear
models, that of neural networks and of course that of all measurable
functions are all locally $\rho$-consistent for some $\rho > 0$. The
guarantees given in the following result are thus general and widely
applicable.
\begin{restatable}[\textbf{$\sH$-consistency bound of $\wt \ell^{\mathrm{comp}}_{\tau,\rho}$}]
{theorem}{BoundCompRhoAdv}
\label{Thm:bound_comp_rho_adv}
Assume that $\sH$ is symmetric and locally $\rho$-consistent. Then,
for any choice of the hyperparameters $\tau, \rho > 0$, any hypothesis
$h \in \sH$, the following inequality holds:
\ifdim\columnwidth= \textwidth
\begin{equation*}
\label{eq:bound_comp_rho_adv}
      \sR_{\ell_{\gamma}}(h)- \sR^*_{\ell_{\gamma}}(\sH)
      \leq \Phi^{\tau}(1)\\
      \paren*{\sR_{\wt \ell^{\mathrm{comp}}_{\tau,\rho}}(h)-\sR^*_{\wt \ell^{\mathrm{comp}}_{\tau,\rho}}(\sH) + \sM_{\wt \ell^{\mathrm{comp}}_{\tau,\rho}}(\sH)} - \sM_{\ell_{\gamma}}(\sH).
\end{equation*}
\else
\begin{align}
\label{eq:bound_comp_rho_adv}
& \sR_{\ell_{\gamma}}(h)- \sR^*_{\ell_{\gamma}}(\sH) \leq\\
& \Phi^{\tau}(1) \bracket*{\sR_{\wt \ell^{\mathrm{comp}}_{\tau,\rho}}(h) -
  \sR^*_{\wt \ell^{\mathrm{comp}}_{\tau,\rho}}(\sH) + \sM_{\wt \ell^{\mathrm{comp}}_{\tau,\rho}}(\sH)}
\!-\!\! \sM_{\ell_{\gamma}}(\sH).\nonumber
\end{align}
\fi
\end{restatable}
The proof is given in
Appendix~\ref{app:deferred_proofs_adv_comp}. Using the inequality
$\ell^{\mathrm{comp}}_{\mathrm{smooth}}\geq \wt
\ell^{\mathrm{comp}}_{\tau,\rho}$ yields the following similar
guarantees for smooth adversarial comp-sum loss under the same
condition of hypothesis sets.
\begin{restatable}[\textbf{Guarantees for smooth adversarial comp-sum losses}]
  {corollary}{SmoothCompRhoAdv}
\label{cor:smooth_comp_rho_adv}
Assume that $\sH$ is symmetric and locally $\rho$-consistent. Then,
for any choice of the hyperparameters $\tau, \rho >0$, any hypothesis
$h \in \sH$, the following inequality
holds:
\begin{align}
\label{eq:smooth_comp_rho_adv}
& \sR_{\ell_{\gamma}}(h)- \sR^*_{\ell_{\gamma}}(\sH) \leq\\
\mspace{-10mu}
& \Phi^{\tau}(1)
\bracket*{\sR_{\ell^{\mathrm{comp}}_{\mathrm{smooth}}}(h)-\sR^*_{\wt \ell^{\mathrm{comp}}_{\tau,\rho}}(\sH) + \sM_{\wt \ell^{\mathrm{comp}}_{\tau,\rho}}(\sH)\! } \!-\!\! \sM_{\ell_{\gamma}}(\sH).\nonumber
\end{align}
\end{restatable}
This is the first $\sH$-consistency bound for the comp-sum loss in the adversarial robustness.
As with the non-adversarial scenario in
Section~\ref{sec:H-consistency-bounds}, the minimizability gaps
appearing in those bounds in Theorem~\ref{Thm:bound_comp_rho_adv} and
Corollary~\ref{cor:smooth_comp_rho_adv} actually equal to zero in most
common cases. More precisely, Theorem \ref{Thm:bound_comp_rho_adv}
guarantees $\sH$-consistency for distributions such that the
minimizability gaps vanish:
\begin{align*}
 \sR_{\ell_{\gamma}}(h)- \sR^*_{\ell_{\gamma}}(\sH) \leq \Phi^{\tau}(1) \bracket*{\sR_{\wt \ell^{\mathrm{comp}}_{\tau,\rho}}(h) -
  \sR^*_{\wt \ell^{\mathrm{comp}}_{\tau,\rho}}(\sH)}.  
\end{align*}
For $\tau \in [0, \infty)$ and $\rho>0$, if the estimation loss
  $(\sR_{\wt \ell^{\mathrm{comp}}_{\tau,\rho}}(h) - \sR_{\wt
    \ell^{\mathrm{comp}}_{\tau,\rho}}^*(\sH))$ is reduced to $\e$,
  then, the adversarial zero-one estimation loss
  $(\sR_{\ell_{\gamma}}(h) - \sR_{\ell_{\gamma}}^*(\sH))$ is bounded
  by $\e$ modulo a multiplicative constant. A similar guarantee
  applies to smooth adversarial comp-sum loss as well. These
  guarantees suggest an adversarial robustness algorithm that consists
  of minimizing a regularized empirical smooth adversarial comp-sum
  loss, $\ell^{\mathrm{comp}}_{\mathrm{smooth}}$. We call this
  algorithm \textsc{adv-comp-sum}. In the next section, we report
  empirical results for \textsc{adv-comp-sum}, demonstrating that it
  significantly outperform the current state-of-the-art loss/algorithm
  \textsc{trades}.
\ignore{Note that \textsc{trades} corresponds to a specific loss
  function \citep{zhang2019theoretically} and it is not clear if
  \textsc{trades} benefits from adversarial $\sH$-consistency bounds.
}

\section{Experiments}
\label{sec:experiments}

We first report empirical results comparing the
performance of comp-sum losses for different values of $\tau$.
Next, we report a series of empirical results comparing
our adversarial robust algorithm \textsc{adv-comp-sum} with
several baselines.

\subsection{Standard Multi-Class Classification}
\label{sec:experiments-non-adv}

We compared comp-sum losses with different values of $\tau$ on
CIFAR-10 and CIFAR-100 datasets
\citep{Krizhevsky09learningmultiple}\ignore{and SVHN
  \citep{Netzer2011}}. All models were trained via Stochastic Gradient
Descent (SGD) with Nesterov momentum \citep{nesterov1983method}, batch
size $1\mathord,024$ and weight decay $1\times 10^{-4}$.  \ignore{We
  use ResNet-$34$ for CIFAR-10 and CIFAR-100, and ResNet-$16$ for
  SVHN.} We used ResNet-$34$ and trained for $200$ epochs using the
cosine decay learning rate schedule \citep{loshchilov2016sgdr} without
restarts. The initial learning rate was selected from $\curl{0.01,
  0.1, 1.0}$; the best model is reported for each surrogate loss. We
report the zero-one classification accuracy of the models and the
standard deviation for three trials.

\begin{table}[h]
\vskip -0.2in
\caption{Zero-one classification accuracy for comp-sum surrogates;
  mean $\pm$ standard deviation over three runs for different $\tau$.}
    \label{tab:comparison-standrad}
    \vskip -0.25in
\begin{center}
\resizebox{0.80\columnwidth}{!}{
    \begin{tabular}{@{\hspace{0pt}}llllll@{\hspace{0pt}}}
      $\tau$ & $0$ & $0.5$ & $1.0$ & $1.5$ & $2.0$ \\
      \midrule
     CIFAR-10  & 87.37  & 90.28 &  92.59 &  92.03  & 90.35 \\
     $\pm$ &  0.57 & 0.10 &   0.10 & 0.08 &  0.24\\
      \midrule
      CIFAR-100 & 57.87  & 65.52& 70.93 & 69.87  & 8.99  \\
      $\pm$ &
      0.60 & 0.34 &  0.34 &  0.39 &  0.98   \\
    \end{tabular}
    }
\end{center}
    \vskip -0.22in
\end{table}

Table~\ref{tab:comparison-standrad} shows that on CIFAR-10 and
CIFAR-100, the logistic loss ($\tau = 1$) outperforms the comp-sum
loss ($\tau = 0.5$) and, by an even larger margin, the sum-exponential
loss ($\tau = 0$). This is consistent with our theoretical analysis
based on $\sH$-consistency bounds in Theorem~\ref{Thm:bound_comp_sum}
since all three losses have the same square-root functional form and
since, by Lemma~\ref{lemma:lemma-compare} and \eqref{eq:wt-Gamma}, the
magnitude of the minimizability gap decreases with $\tau$.

Table~\ref{tab:comparison-standrad} also shows that on CIFAR-10 and
CIFAR-100, the logistic loss ($\tau = 1$) and the generalized
cross-entropy loss ($\tau = 1.5$) achieve relatively close results
that are clearly superior to that of mean absolute error loss
($\tau=2$). This empirical observation agrees with our theoretically
analysis based on their $\sH$-consistency bounds
(Theorem~\ref{Thm:bound_comp_sum}): by
Lemma~\ref{lemma:lemma-compare}, the minimizability gap of $\tau=1.5$
and $\tau=2$ is smaller than that of $\tau = 1$; however, by
\eqref{eq:wt-Gamma}, the dependency of the multiplicative constant on
the number of classes appears for $\tau=1.5$ in the form of
$\sqrt{n}$, which makes the generalized cross-entropy loss less
favorable, and for $\tau = 2$ in the form of $n$, which makes the mean
absolute error loss least favorable. Another reason for the inferior
performance of the mean absolute error loss ($\tau=2$) is that, as
observed in our experiments, it is difficult to optimize in practice,
using deep neural networks on complex datasets. This has also been
previously reported by \citet{zhang2018generalized}. In fact, the mean
absolute error loss can be formulated as an $\ell_1$-distance and is
therefore not smooth; but it has the advantage of robustness, as shown
in \citep{ghosh2017robust}.

\begin{table*}[t]
\vskip -0.1in
\caption{Clean accuracy and robust accuracy under
  PGD$^{40}_{\mathrm{margin}}$ and AutoAttack; mean $\pm$ standard
  deviation over three runs for both \textsc{adv-comp-sum} and the
  state-of-the-art \textsc{trades} in
  \citep{gowal2020uncovering}. Accuracies of some well-known
  adversarial defense models are included for
  completeness. \textsc{adv-comp-sum} significantly outperforms
  \textsc{trades} for both robust and clean accuracy in all the
  settings.}
    \label{tab:comparison}
    \vskip -0.2in
\begin{center}
    \resizebox{.66\textwidth}{!}{
    \begin{tabular}{@{\hspace{0pt}}lllll@{\hspace{0pt}}}
      Method & Dataset & Clean & PGD$^{40}_{\mathrm{margin}}$ & AutoAttack \\
    \midrule
    \citet{gowal2020uncovering} (WRN-70-16) & \multirow{10}{*}{CIFAR-10} & 85.34 $\pm$ 0.04 & 57.90 $\pm$ 0.13 & 57.05 $\pm$ 0.17\\
    \textbf{\textsc{adv-comp-sum} (WRN-70-16)} &  & \textbf{86.16 $\pm$ 0.16} & \textbf{59.35 $\pm$ 0.07} &  \textbf{57.77 $\pm$ 0.08} \\
    \citet{gowal2020uncovering} (WRN-34-20) & &  85.21 $\pm$ 0.16 & 57.54 $\pm$ 0.18 & 56.70 $\pm$ 0.14 \\
    \textbf{\textsc{adv-comp-sum} (WRN-34-20)} &  &  \textbf{85.59 $\pm$ 0.17} & \textbf{58.92 $\pm$ 0.06} &  \textbf{57.41 $\pm$ 0.06}\\
    \citet{gowal2020uncovering} (WRN-28-10) & &  84.33 $\pm$ 0.18 & 55.92 $\pm$ 0.20 & 55.19 $\pm$ 0.23 \\
    \textbf{\textsc{adv-comp-sum} (WRN-28-10)} & & \textbf{84.50 $\pm$ 0.33} & \textbf{57.28 $\pm$ 0.05} &  \textbf{55.79 $\pm$ 0.06}\\
    \cmidrule{1-1} \cmidrule{3-5}
    \citet{pang2020bag} (WRN-34-20)&  &  86.43& \NA&54.39\\
    \citet{DBLP:conf/icml/RiceWK20} (WRN-34-20)&  &  85.34& \NA&53.42 \\
    \citet{wu2020adversarial} (WRN-34-10)&  &  85.36& \NA&56.17 \\
    \citet{qin2019adversarial} (WRN-40-8)&  &  86.28& \NA&52.84 \\
    \midrule
    \citet{gowal2020uncovering} (WRN-70-16) & \multirow{2}{*}{CIFAR-100} & 60.56 $\pm$ 0.31 & 31.39 $\pm$ 0.19 & 29.93 $\pm$ 0.14\\
    \textbf{\textsc{adv-comp-sum} (WRN-70-16)}  & &  \textbf{63.10 $\pm$ 0.24} & \textbf{33.76 $\pm$ 0.18} & \textbf{31.05 $\pm$ 0.15}\\
    \midrule
    \citet{gowal2020uncovering} (WRN-34-20) & \multirow{2}{*}{SVHN} & 93.03 $\pm$ 0.13 & 61.01 $\pm$ 0.16 & 57.84 $\pm$ 0.19\\
    \textbf{\textsc{adv-comp-sum} (WRN-34-20)}  & &  \textbf{93.98 $\pm$ 0.12} & \textbf{62.97 $\pm$ 0.05} &  \textbf{58.13 $\pm$ 0.12}\\
    \end{tabular}
    }
\end{center}
    \vskip -0.25in
\end{table*}

\subsection{Adversarial Multi-Class Classification}

Here, we report empirical results for our adversarial robustness
algorithm \textsc{adv-comp-sum} on CIFAR-10, CIFAR-100
\citep{Krizhevsky09learningmultiple} and SVHN \citep{Netzer2011}
datasets\ignore{, showing that it outperforms the current
  state-of-the-art algorithm, \textsc{trades},}. No generated data or
extra data was used.

\textbf{Experimental settings.}  We followed exactly the experimental
settings of \citet{gowal2020uncovering} and adopted precisely the same
training procedure and neural network architectures, which are
WideResNet (WRN) \citep{zagoruyko2016wide} with SiLU activations
\citep{hendrycks2016gaussian}. Here, WRN-$n$-$k$ denotes a residual
network with $n$ convolutional layers and a widening factor $k$. For
CIFAR-10 and CIFAR-100, the simple data augmentations, 4-pixel padding
with $32 \times 32$ random crops and random horizontal flips, were
applied. We used 10-step Projected Gradient-Descent (PGD) with random
starts to generate training attacks. All models were trained via
Stochastic Gradient Descent (SGD) with Nesterov momentum
\citep{nesterov1983method}, batch size $1\mathord,024$ and weight
decay $5\times 10^{-4}$. We trained for $400$ epochs using the cosine
decay learning rate schedule \citep{loshchilov2016sgdr} without
restarts. The initial learning rate is set to $0.4$. We used model
weight averaging \citep{DBLP:conf/uai/IzmailovPGVW18} with decay rate
$0.9975$. For \textsc{trades}, we adopted exactly the same setup as
\citet{gowal2020uncovering}. For our smooth adversarial comp-sum
losses, we set both $\rho$ and $\nu$ to $1$ by default. In practice,
they can be selected by cross-validation and that could potentially
lead to better performance. The per-epoch computational cost of our
method is similar to that of \textsc{trades}.

\textbf{Evaluation.} We used early stopping on a held-out validation
set of $1\mathord,024$ samples by evaluating its robust accuracy
throughout training with 40-step PGD on the margin loss, denoted by
PGD$^{40}_{\mathrm{margin}}$, and selecting the best check-point
\citep{DBLP:conf/icml/RiceWK20}.
We report the \emph{clean accuracy}, that is the standard
classification accuracy on the test set, and the robust accuracy with
$\ell_{\infty}$-norm perturbations bounded by $\gamma = 8/255$ under
PGD attack, measured by PGD$^{40}_{\mathrm{margin}}$ on the full test
set, as well as under AutoAttack \citep{croce2020reliable} ({\small
  \url{https://github.com/fra31/auto-attack}}), the state-of-the-art
attack for measuring empirically adversarial robustness. We averaged
accuracies over three runs and report the standard deviation for both
\textsc{adv-comp-sum} and \textsc{trades}, reproducing the results
reported for \textsc{trades} in \citep{gowal2020uncovering}.

\textbf{Results}.  Table~\ref{tab:comparison} shows that
\textsc{adv-comp-sum} outperforms \textsc{trades} on CIFAR-10 for all
the neural network architectures adopted (WRN-70-16, WRN-34-20 and
WRN-28-10). Here, \textsc{adv-comp-sum} was implemented with
$\tau=0.4$. Other common choices of $\tau$ yield similar results,
including $\tau=1$ (logistic loss). In all the settings, robust
accuracy under AutoAttack is higher by at least 0.6\% for
\textsc{adv-comp-sum}, by at least 1.36\% under the
PGD$^{40}_{\mathrm{margin}}$ attack.

It is worth pointing out that the improvement in robustness accuracy
for our models does not come at the expense of a worse clean accuracy
than \textsc{trades}. In fact, \textsc{adv-comp-sum} consistently
outperforms \textsc{trades} for the clean accuracy as well. For the
largest model WRN-70-16, the improvement is over 0.8\%. For
completeness, we also include in Table~\ref{tab:comparison} the
results for some other well-known adversarial defense
models. \textsc{adv-comp-sum} with the smallest model WRN-28-10
surpasses \citep{pang2020bag, DBLP:conf/icml/RiceWK20,
  qin2019adversarial}. \citep{wu2020adversarial} is significantly
outperformed by \textsc{adv-comp-sum} with a slightly larger model
WRN-34-20, by more than 1.2\% in the robust accuracy and also in the
clean accuracy.

To show the generality of our approach, we carried out experiments
with other datasets, including CIFAR-100 and SVHN. For WRN-70-16 on
CIFAR-100, \textsc{adv-comp-sum} outperforms \textsc{trades} by 1.12\%
in the robust accuracy and 2.54\% in the clean accuracy. For WRN-34-20
on SVHN, \textsc{adv-comp-sum} also outperforms \textsc{trades} by
0.29\% in the robust accuracy and 0.95\% in the clean accuracy.

Let us underscore that outperforming the state-of-the-art results of
\citet{gowal2020uncovering} in the same scenario and without resorting
to additional unlabeled data has turned out to be very challenging:
despite the large research emphasis on this topic in the last several
years and the many publications, none was reported to surpass that
performance, using an alternative surrogate loss.\ignore{ Thus, the
  improvements we report both in clean accuracy and robust accuracy
  with respect to the results of \citep{gowal2020uncovering} are
  significant.}

\section{Discussion}
\label{sec:future-work}

\textbf{Applications of $\sH$-consistency bounds}.  Given a hypothesis
set $\sH$, our quantitative $\sH$-consistency bounds can help select
the most favorable surrogate loss, which depends on (i) the functional
form of the $\sH$-consistency bound: for instance, the bound for the
mean absolute error loss exhibits a linear dependency, while that of
the logistic loss and generalized cross-entropy losses exhibit a
square-root dependency, resulting in a less favorable convergence
rate; (ii) the smoothness of the loss and, more generally, its
optimization properties; for example, the mean absolute error loss is
less smooth than the logistic loss, and surrogate losses with more
favorable bounds may lead to more challenging optimizations; in fact,
the zero-one loss serves as its own surrogate with the tightest bound
for any hypothesis set, but is known to result in NP-complete
optimization problems for many common choices of $\sH$; (iii)
approximation properties of the surrogate loss function: for instance,
given a choice of $\sH$, the minimizability gap for a surrogate loss
may be more or less favorable; (iv) the dependency of the
multiplicative constant on the number of classes: for example, the
linear dependency of $n$ in the bound for the mean absolute error loss
makes it less favorable than the logistic loss.

Another application is the derivation of generalization bounds for
surrogate loss minimizers (see Theorem ~\ref{th:genbound}), expressed
in terms of the quantities discussed above.

\textbf{Concurrent work}. The concurrent and independent study of
\citet{zheng2023revisiting} also provides an $\sH$-consistency bound
for the logistic loss. Their bound holds for the special case of $\sH$
being a constrained linear hypothesis set, subject to an additional
assumption on the distribution. In contrast, our bounds do not require
any distributional assumption. However, it should be noted that our
results are only applicable to complete hypothesis sets.  In upcoming
work, we present $\sH$-consistency bounds for non-complete hypothesis
sets and arbitrary distributions.

\textbf{Future work}. In addition to the extension to non-complete
hypothesis sets just mentioned, it would be valuable to investigate
the application or generalization of $\sH$-consistency bounds in
scenarios involving noisy labels
\citep{ghosh2017robust,zhang2018generalized}.
For comp-sum losses, this paper focuses on the case where $\Phi_2$ is
the exponential loss and $\Phi_1$ is based on \eqref{eq:Phi1}. This
includes the cross-entropy loss (or logistic loss), generalized
cross-entropy, the mean absolute error and other cross-entropy-like
functions, which are the most widely used ones in the family of
comp-sum losses. The study of other such loss functions and the
comparison with other families of multi-class loss functions
\citep{AwasthiMaoMohriZhong2022multi} is left to the future work.
Although our algorithm demonstrates improvements over the
current state-of-the-art technique, adversarial robustness remains a
challenging problem. A key issue seems to be that of generalization
for complex families of neural networks (see for example
\citep*{awasthi2020adversarial}). A more detailed study of that
problem might help enhance the performance of our algorithm.
Finally, in addition to their immediate implications, our results and
techniques have broader applications in analyzing surrogate losses and
algorithms across different learning scenarios. For instance, they can
be used in the context of ranking, as demonstrated in recent work by
\citet*{MaoMohriZhong2023}. Furthermore, they can be extended to
address the challenges of learning with abstention
\citep*{cortes2016learning,cortes2016boosting}. Additionally, our
findings can be valuable in non-i.i.d.\ learning settings, such as
drifting \citep{MohriMunozMedina2012} or time series prediction
\citep{KuznetsovMohri2018,KuznetsovMohri2020}.

\section{Conclusion}
\label{sec:conclusion}

We presented a detailed analysis of the theoretical properties of a
family of surrogate losses that includes the logistic loss (or
cross-entropy with the softmax). These are more precise and more
informative guarantees than Bayes consistency since they are
non-asymptotic and specific to the hypothesis set used. Our bounds are
tight and can be made more explicit, when combined with our analysis
of minimizability gaps. These inequalities can help compare different
surrogate losses and evaluate their advantages in different scenarios.
We showcased one application of this analysis by extending comp-sum
losses to the adversarial robustness setting, which yields principled
surrogate losses and algorithms for that scenario. We believe that our
analysis can be helpful to the design of algorithms in many other
scenarios.

\section*{Acknowledgements}

Part of the work of A.~Mao and Y.~Zhong was done during their
internship at Google Research.

\bibliography{comp}
\bibliographystyle{icml2023}

\newpage
\appendix
\onecolumn

\renewcommand{\contentsname}{Contents of Appendix}
\tableofcontents
\addtocontents{toc}{\protect\setcounter{tocdepth}{3}} 
\clearpage

\newpage
\section{Related work}
\label{app:realted-work}

\textbf{Consistency guarantees}. The concept of Bayes consistency, or
the related one of classification calibration
\citep{Zhang2003,bartlett2006convexity,steinwart2007compare,
  MohriRostamizadehTalwalkar2018}, have been extensively explored in a
wide range of contexts, including multi-class classification
\citep{zhang2004statistical,chen2006consistency,
  chen2006consistency2,tewari2007consistency,liu2007fisher,
  dogan2016unified,wang2020weston,ramaswamy2012classification,
  narasimhan2015consistent,agarwal2015consistent,williamson2016composite,
  ramaswamy2016convex,finocchiaro2019embedding,frongillo2021surrogate,finocchiaro2022embedding,wang2023classification}, multi-label
classification
\citep{gao2011consistency,dembczynski2012consistent,zhang2020convex},
learning with rejection
\citep{cortes2016learning,cortes2016boosting,ramaswamy2015consistent,
  cao2022generalizing}, ranking
\citep{duchi2010consistency,ravikumar2011ndcg,ramaswamy2013convex,
  gao2015consistency,uematsu2017theoretically,MaoMohriZhong2023},
cost-sensitive classification
\citep{pires2013cost,pires2016multiclass}, structured prediction
\citep{ciliberto2016consistent,osokin2017structured,blondel2019structured},
top-$k$ classification \citep{thilagar2022consistent},  structured abstain problem \citep{nueve2022structured} and ordinal
regression \citep{pedregosa2017consistency}. Bayes-consistency does
not supply any information about learning with a typically restricted
hypothesis set. Another line of research focuses on realizable
$\sH$-consistency guarantees
\citep{long2013consistency,zhang2020bayes,KuznetsovMohriSyed2014},
which provides hypothesis set-specific consistency guarantees under
the assumption that the underlying distribution is
$\sH$-realizable. However, none of these guarantees is informative for
approximate minimizers (non-asymptotic guarantee) since
convergence could be arbitrarily slow.

The concept of $\sH$-consistency bounds was first introduced by
\citet{awasthi2022Hconsistency} in binary classification and
subsequently extended by \citet{AwasthiMaoMohriZhong2022multi} to the
scenario of multi-class classification. Such guarantees are both
non-asymptotic and hypothesis set-specific. This paper presents the
first tight $\sH$-consistency bounds for the \emph{comp-sum losses},
that includes cross-entropy (or logistic loss), generalized
cross-entropy, the mean absolute error and other loss
cross-entropy-like functions.

The concurrent and independent study of
\citet{zheng2023revisiting} also provides an $\sH$-consistency bound
for the logistic loss. Their bound holds for the special case of $\sH$
being a constrained linear hypothesis set, subject to an additional
assumption on the distribution. In contrast, our bounds do not require
any distributional assumption. However, it should be noted that our
results are only applicable to complete hypothesis sets.  In upcoming
work, we present $\sH$-consistency bounds for non-complete hypothesis
sets and arbitrary distributions.

\textbf{Adversarial robustness}. From a theoretical perspective, there
has been significant dedication towards providing guarantees for
adversarial robustness
\citep{szegedy2013intriguing,biggio2013evasion,goodfellow2014explaining,
  madry2017towards,carlini2017towards}, including PAC Learnability
\citep{feige2015learning,feige2018robust,montasser2019vc,
  attias2022adversarially,ashtiani2020black,bhattacharjee2021sample,
  cullina2018pac,dan2020sharp,montasser2020efficiently,
  montasser2020reducing,
  attiascharacterization,diakonikolas2020complexity,
  montasser2021adversarially,kontorovich2021fat,
  montasser2022transductive}, robust generalization
\citep{khim2018adversarial,attias2019improved,xing2021adversarially,
  yin2019rademacher,schmidt2018adversarially,awasthi2020adversarial,
  attias2022improved,xiaostability,viallard2021pac,bubeck2021universal,li2023achieving},
adversarial examples
\citep{bubeck2018adversarial2,bubeck2019adversarial,
  bartlett2021adversarial,bubeck2021single}, consistency guarantees and optimal adversarial classifiers
\citep{pmlr-v125-bao20a,awasthi2021calibration,awasthi2021finer,awasthi2021existence,awasthi2022Hconsistency,AwasthiMaoMohriZhong2022multi,meunier2022towards,li2023achieving,frank2023adversarial} and optimization
\citep{awasthi2019robustness,robey2021adversarial,awasthi2023dc}.

From an algorithmic standpoint, numerous defense strategies have been
proposed historically, including adversarial surrogate loss functions
\citep{kurakin2016adversarial,madry2017towards,tsipras2018robustness,
  kannan2018adversarial,zhang2019theoretically,wang2020improving,
  levidomain,jin2022enhancing,awasthi2023theoretically},
curriculum and adaptive attack methods
\citep{cai2018curriculum,wang2019convergence,zhang2020attacks,
  dingmma,cheng2022cat},
efficient adversarial training
\citep{shafahi2019adversarial,zhang2019you,wong2020fast,
  andriushchenko2020understanding},
ensemble techniques
\citep{tramer2018ensemble,pang2019improving,yang2020dverge,guo2022fast},
unlabeled data
\citep{carmon2019unlabeled,alayrac2019labels,zhai2019adversarially},
data augmentation \citep{rebuffi2021fixing,rebuffi2021data}, neural
network architectures
\citep{xieintriguing,xie2019feature,liu2020towards,guo2020meets},
weight averaging/perturbation
\citep{gowal2020uncovering,wu2020adversarial,tsai2021formalizing,
  yu2018interpreting,prabhu2019understanding}
and other techniques
\citep{zhang2019defense,qin2019adversarial,goldblum2020adversarially,
  songimproving,pang2020boosting,lee2020adversarial,qian2021improving}.

This paper introduces a new family of loss functions, \emph{smooth
adversarial comp-sum losses}, derived from their comp-sum counterparts
by adding in a related smooth term.  We show that these loss functions
are beneficial in the adversarial setting by proving that they admit
$\sH$-consistency bounds.  This leads to new adversarial robustness
algorithms that consist of minimizing a regularized smooth adversarial
comp-sum loss. We report the results of a series of experiments
demonstrating that our algorithms outperform the current
state-of-the-art, while also achieving a superior non-adversarial
accuracy.

\section{Proofs of \texorpdfstring{$\sH$}{H}-consistency bounds
  for comp-sum losses (Theorem~\ref{Thm:bound_comp_sum})
  and tightness (Theorem~\ref{Thm:tightness})}
\label{app:bound_comp-sum}

To begin with the proof, we first introduce some notation. We denote
by $p(x, y) = \sD(Y = y \!\mid\! X = x)$ the conditional probability
of $Y=y$ given $X = x$. The generalization error for a surrogate loss
can be rewritten as $ \sR_{\ell}(h) = \mathbb{E}_{X}
\bracket*{\sC_{\ell}(h, x)} $, where $\sC_{\ell}(h, x)$ is the
conditional $\ell$-risk, defined by
\begin{align*}
\sC_{\ell}(h, x) = \sum_{y\in \sY} p(x, y) \ell(h, x, y).
\end{align*}
We denote by $\sC_{\ell}^*(\sH,x) = \inf_{h\in
  \sH}\sC_{\ell}(h, x)$ the minimal conditional
$\ell$-risk. Then, the minimizability gap can be rewritten as follows:
\begin{align*}
\sM_{\ell}(\sH)
 = \sR^*_{\ell}(\sH) - \mathbb{E}_{X} \bracket* {\sC_{\ell}^*(\sH, x)}.
\end{align*}
We further refer to $\sC_{\ell}(h, x)-\sC_{\ell}^*(\sH,x)$ as
the calibration gap and denote it by $\Delta\sC_{\ell,\sH}(h, x)$. 

For any $h \in \sH$ and $x\in \sX$, by the symmetry and completeness
of $\sH$, we can always find a family of hypotheses $\curl*{\ov
  h_{\mu} \colon \mu \in \mathbb{R}}\subset \sH$ such that
$h_{\mu}(x,\cdot)$ take the following values:
\begin{align}
\label{eq:value-h-mu}
\ov h_{\mu}(x,y) = 
\begin{cases}
  h(x, y) & \text{if $y \not \in \curl*{y_{\max}, \hh(x)}$}\\
  \log\paren*{\exp\bracket*{h(x, y_{\max})} + \mu} & \text{if $y = \hh(x)$}\\
  \log\paren*{\exp\bracket*{h(x,\hh(x))} -\mu} & \text{if $y = y_{\max}$}.
\end{cases} 
\end{align}
Note that the hypotheses $\ov h_{\mu}$ has the following property:
\begin{align}
\label{eq:property-h-mu}
\sum_{y\in \sY}e^{h(x, y)} = \sum_{y\in \sY}e^{\ov h_{\mu}(x, y)},\, \forall \mu \in \mathbb{R}.
\end{align}
\begin{restatable}
  {lemma}{LemmaSup}
\label{lemma:lemma-sup}
Assume that $\sH$ is symmetric and complete. Then, for any $h\in \sX$
and $x\in \sX$, the following equality holds:
\small
\begin{align*}
& \sC_{\ell_{\tau}^{\mathrm{comp}}}(h, x) - \inf_{\mu \in \Rset}\sC_{\ell_{\tau}^{\mathrm{comp}}}(\ov h_{\mu},x)\\
& = \sup_{\mu\in \Rset} \bigg\{p(x,y_{\max})\paren*{ \Phi^{\tau}\paren*{\frac{\sum_{y'\in \sY} e^{h(x, y')}}{e^{h(x, y_{\max})}}-1}-\Phi^{\tau}\paren*{\frac{\sum_{y'\in \sY} e^{h(x, y')}}{e^{h(x, \hh(x))-\mu}}-1}}\\
& \mspace{60mu} +  p(x,\hh(x))\paren*{ \Phi^{\tau}\paren*{\frac{\sum_{y'\in \sY} e^{h(x, y')}}{e^{h(x, \hh(x))}}-1}-\Phi^{\tau}\paren*{\frac{\sum_{y'\in \sY} e^{h(x, y')}}{e^{h(x, y_{\max})+\mu}}-1}} \bigg\}\\
& = \begin{cases}
\frac{1}{\tau-1}\bracket*{\sum_{y'\in \sY}e^{h(x,y')}}^{1-\tau}\bracket*{\frac{\paren*{p(x,y_{\max})^{\frac1{2 - \tau }}+p(x,\hh(x))^{\frac1{2 - \tau }}}^{2-\tau}}{\paren*{e^{h(x,y_{\max})} + e^{h(x,\hh(x))}}^{1 - \tau }}-\frac{p(x,y_{\max})}{\paren*{e^{h(x,y_{\max})}}^{1 - \tau }} - \frac{p(x,\hh(x))}{\paren*{e^{h(x,\hh(x))}}^{1 - \tau }}} & \tau \in [0,2)/\curl*{1}\\
p(x,y_{\max})\log\bracket*{\frac{\paren*{e^{h(x,y_{\max})} + e^{h(x,\hh(x))}}p(x,y_{\max})}{e^{h(x,y_{\max})}\paren*{p(x,y_{\max})+p(x,\hh(x))}}}+p(x,\hh(x))\log\bracket*{\frac{\paren*{e^{h(x,y_{\max})} +e^{h(x,\hh(x))}}p(x,\hh(x))}{ e^{h(x,\hh(x))}\paren*{p(x,y_{\max})+p(x,\hh(x))}}} & \tau =1 \\
\frac{1}{\tau-1}\bracket*{\sum_{y'\in \sY}e^{h(x,y')}}^{1-\tau}\bracket*{\frac{p(x,y_{\max})}{\paren*{e^{h(x,y_{\max})} + e^{h(x,\hh(x))}}^{1 - \tau }}-\frac{p(x,y_{\max})}{\paren*{e^{h(x,y_{\max})}}^{1 - \tau }} - \frac{p(x,\hh(x))}{\paren*{e^{h(x,\hh(x))}}^{1 - \tau }} } & \tau \in [2,\plus \infty).
\end{cases}
\end{align*}
\normalsize
\end{restatable}
\begin{proof}
For the comp-sum loss $\ell_{\tau}^{\rm{comp}}$, the conditional $\ell_{\tau}^{\rm{comp}}$-risk can be expressed as follows:
\begin{align*}
\sC_{\ell_{\tau}^{\rm{comp}}}(h, x)
& = \sum_{y\in \sY} p(x,y) \ell_{\tau}^{\rm{comp}}(h, x, y)\\
 & = \sum_{y\in \sY} p(x,y) \Phi^{\tau}\paren*{\sum_{y'\in \sY} e^{h(x, y') - h(x, y)}-1}\\
& = p(x,y_{\max}) \Phi^{\tau}\paren*{\sum_{y'\in \sY} e^{h(x, y') - h(x, y_{\max})}-1}+  p(x,\hh(x)) \Phi^{\tau}\paren*{\sum_{y'\in \sY} e^{h(x, y') - h(x, \hh(x))}-1}\\
& +\sum_{y\notin \curl*{y_{\max},\hh(x)}}p(x,y) \Phi^{\tau}\paren*{\sum_{y'\in \sY} e^{h(x, y') - h(x, y)}-1}.
\end{align*}
Therefore, by \eqref{eq:value-h-mu} and \eqref{eq:property-h-mu}, we
obtain the first equality. The second equality can be obtained by
taking the derivative with respect to $\mu$.
\end{proof}

\begin{restatable}
  {lemma}{LemmaInf}
\label{lemma:lemma-inf}
Assume that $\sH$ is symmetric and complete. Then, for any $h\in \sX$
and $x\in \sX$, the following equality holds
\begin{align*}
& \inf_{h\in \sH}\paren*{\sC_{\ell_{\tau}^{\mathrm{comp}}}(h, x) - \inf_{\mu \in \Rset}\sC_{\ell_{\tau}^{\mathrm{comp}}}(\ov h_{\mu},x)}\\
& = 
\begin{cases}
\frac{2^{2-\tau}}{1-\tau}\bracket*{\frac{p(x,y_{\max})+p(x,\hh(x))}{2}-\bracket*{\frac{p(x,y_{\max})^{\frac1{2 - \tau }}+p(x,\hh(x))^{\frac1{2 - \tau }}}{2}}^{2 - \tau }} & \tau \in [0,1)\\
p(x,y_{\max})\log\bracket*{\frac{2p(x,y_{\max})}{p(x,y_{\max})+p(x,\hh(x))}} + p(x,\hh(x))\log\bracket*{\frac{2p(x,\hh(x))}{p(x,y_{\max})+p(x,\hh(x))}} & \tau =1 \\
\frac{2}{(\tau-1)n^{\tau-1}}\paren*{\bracket*{\frac{p(x,y_{\max})^{\frac1{2-\tau}}+p(x,\hh(x))^{\frac1{2-\tau}}}{2}}^{2-\tau}-\frac{p(x,y_{\max})+p(x,\hh(x))}{2}} & \tau \in (1,2)\\
\frac{1}{(\tau-1)n^{\tau-1}}
\paren*{p(x,y_{\max}) - p(x,\hh(x))} & \tau \in [2,\plus \infty).
\end{cases}
\end{align*}
\end{restatable}
\begin{proof}
By using Lemma~\ref{lemma:lemma-sup} and taking infimum with respect to $e^{h(x,1)},\ldots, e^{h(x,n)}$, the equality is proved directly.
\end{proof}
Let $\alpha=p(x,y_{\max})+p(x,\hh(x))\in [0,1]$ and
$\beta=p(x,y_{\max}-p(x,\hh(x)))\in [0,1]$. Then, using the fact that
$p(x,y_{\max})= \frac{\alpha+\beta}{2}$ and
$p(x,\hh(x))= \frac{\alpha-\beta}{2}$, we can rewrite $\inf_{h\in
  \sH}\paren*{\sC_{\ell_{\tau}^{\mathrm{comp}}}(h, x) - \inf_{\mu \in
    \Rset}\sC_{\ell_{\tau}^{\mathrm{comp}}}(\ov h_{\mu},x)}$ as
\begin{align}
\label{eq:Psi-tau}
\inf_{h\in \sH}\paren*{\sC_{\ell_{\tau}^{\mathrm{comp}}}(h, x) - \inf_{\mu \in \Rset}\sC_{\ell_{\tau}^{\mathrm{comp}}}(\ov h_{\mu},x)} = \Psi_{\tau}(\alpha,\beta)= 
\begin{cases}
\frac{2^{1-\tau}}{1-\tau}\bracket*{\alpha-\bracket*{\frac{\paren*{\alpha+\beta}^{\frac1{2 - \tau }}+\paren*{\alpha-\beta}^{\frac1{2 - \tau }}}{2}}^{2 - \tau }} & \tau \in [0,1)\\
\frac{\alpha+\beta}{2}\log\bracket*{\frac{\alpha+\beta}{\alpha}} + \frac{\alpha-\beta}{2}\log\bracket*{\frac{\alpha-\beta}{\alpha}} & \tau =1 \\
\frac{1}{(\tau-1)n^{\tau-1}}\paren*{\bracket*{\frac{\paren*{\alpha+\beta}^{\frac1{2 - \tau }}+\paren*{\alpha-\beta}^{\frac1{2 - \tau }}}{2}}^{2 - \tau}-\alpha} & \tau \in (1,2)\\
\frac{1}{(\tau-1)n^{\tau-1}}\,
\beta & \tau \in [2,\plus \infty).
\end{cases}
\end{align}
By taking the partial derivative of $\Psi_{\tau}(\alpha,\cdot)$ with respect to $\alpha$ and analyzing the minima, we obtain the following result.
\begin{restatable}
  {lemma}{LemmaPsiTau}
\label{lemma:Psi-tau}
For any $\tau \in [0,\plus\infty)$ and $\alpha\in [0,1]$, the following inequality holds for any $\beta\in [0,1]$,
\begin{align*}
\Psi_{\tau}(\alpha,\beta)\geq \Psi_{\tau}(1,\beta)= \sT_{\tau}(\beta)= \begin{cases}
\frac{2^{1-\tau}}{1-\tau}\bracket*{1 -\bracket*{\frac{\paren*{1 + \beta}^{\frac1{2 - \tau }} +  \paren*{1 - \beta}^{\frac1{2 - \tau }}}{2}}^{2 - \tau }} & \tau \in [0,1)\\
\frac{1+\beta}{2}\log\bracket*{1+\beta} + \frac{1-\beta}{2}\log\bracket*{1-\beta} & \tau =1 \\
\frac{1}{(\tau-1)n^{\tau-1}}\bracket*{\bracket*{\frac{\paren*{1 + \beta}^{\frac1{2 - \tau }} +  \paren*{1 - \beta}^{\frac1{2 - \tau }}}{2}}^{2 - \tau }-1} & \tau \in (1,2)\\
\frac{1}{(\tau-1)n^{\tau-1}}\,
\beta & \tau \in [2,\plus \infty).
\end{cases}
\end{align*}
\end{restatable}
We denote by $\sT_{\tau}(\beta)= \Psi_{\tau}(1,\beta)$ and call it the
$\sH$-consistency comp-sum transformation, and denote by
$\Gamma_{\tau}$ the inverse of $\sT_{\tau}$: $\Gamma_{\tau}(t) =
\sT_{\tau}^{-1}(t)$. We then present the proofs of
Theorem~\ref{Thm:bound_comp_sum} and Theorem~\ref{Thm:tightness} in
the below.  \ignore{By using standard inequalities, $\Gamma_{\tau}$
  can be upper bounded as follows
\begin{align*}
\Gamma_{\tau}(t)\leq \begin{cases}
\sqrt{2^{\frac{\tau}{2-\tau}}(2-\tau) t} & \tau\in [0,1)\\
\sqrt{2^{2-\tau}(2-\tau)n^{\frac{(\tau-1)(4-\tau)}{2-\tau}} t } & \tau\in [1,2) \\
\ignore{\sqrt{\frac{2n^{\frac{2\tau-2}{2-\tau}}}{2-\tau} t } & \tau\in [1,2) \\}
(\tau - 1) n^{\tau - 1} t & \tau \in [2,\plus \infty).
\end{cases}
\end{align*}}

\BoundCompSum*
\begin{proof}
Using previous lemmas, we can lower bound the calibration gap of
comp-sum losses as follows, for any $h\in \sH$,
\begin{align*}
& \sC_{\ell_{\tau}^{\mathrm{comp}}}(h, x) - \sC^*_{\ell_{\tau}^{\mathrm{comp}}}\paren*{\sH,x} \\
& \geq \sC_{\ell_{\tau}^{\mathrm{comp}}}(h, x) - \inf_{\mu \in \Rset}\sC_{\ell_{\tau}^{\mathrm{comp}}}(\ov h_{\mu},x) \\
& \geq \inf_{h\in \sH}\paren*{\sC_{\ell_{\tau}^{\mathrm{comp}}}(h, x) - \inf_{\mu \in \Rset}\sC_{\ell_{\tau}^{\mathrm{comp}}}(\ov h_{\mu},x)}\\
& =
\begin{cases}
\frac{2^{2-\tau}}{1-\tau}\bracket*{\frac{p(x,y_{\max})+p(x,\hh(x))}{2}-\bracket*{\frac{p(x,y_{\max})^{\frac1{2 - \tau }}+p(x,\hh(x))^{\frac1{2 - \tau }}}{2}}^{2 - \tau }} & \tau \in [0,1)\\
p(x,y_{\max})\log\bracket*{\frac{2p(x,y_{\max})}{p(x,y_{\max})+p(x,\hh(x))}} + p(x,\hh(x))\log\bracket*{\frac{2p(x,\hh(x))}{p(x,y_{\max})+p(x,\hh(x))}} & \tau =1 \\
\frac{2}{(\tau-1)n^{\tau-1}}\paren*{\bracket*{\frac{p(x,y_{\max})^{\frac1{2-\tau}}+p(x,\hh(x))^{\frac1{2-\tau}}}{2}}^{2-\tau}-\frac{p(x,y_{\max})+p(x,\hh(x))}{2}} & \tau \in (1,2)\\
\frac{1}{(\tau-1)n^{\tau-1}}
\paren*{p(x,y_{\max}) - p(x,\hh(x))} & \tau \in [2,\plus \infty)
\end{cases}
\tag{By Lemma~\ref{lemma:lemma-inf}}\\
& \geq \sT_{\tau}\paren*{p(x,y_{\max}) - p(x,\hh(x))} \tag{By \eqref{eq:Psi-tau} and Lemma~\ref{lemma:Psi-tau}}\\
& = \sT_{\tau}\paren*{\sC_{\ell_{0-1}}(h, x) - \sC^*_{\ell_{0-1}}\paren*{\sH,x}} \tag{by \citep[Lemma~ 3]{AwasthiMaoMohriZhong2022multi}}
\end{align*}
Therefore, taking $\sP$ be the set of all distributions, $\sH$ be the
symmetric and complete hypothesis set, $\e=0$ and
$\Psi(\beta)= \sT_{\tau}(\beta)$ in \citep[Theorem~
  4]{AwasthiMaoMohriZhong2022multi}, or, equivalently, $\Gamma(t) =
\Gamma_{\tau}(t)$ in \citep[Theorem~
  5]{AwasthiMaoMohriZhong2022multi}, we obtain for any hypothesis
$h\in\sH$ and any distribution,
\begin{align*}
\sR_{\ell_{0-1}}(h)-\sR^*_{\ell_{0-1}}(\sH)\leq \Gamma_{\tau}\paren*{\sR_{\ell_{\tau}^{\mathrm{comp}}}(h)-\sR^*_{\ell_{\tau}^{\mathrm{comp}}}(\sH)+\sM_{\ell_{\tau}^{\mathrm{comp}}}(\sH)}-\sM_{\ell_{0-1}}(\sH).
\end{align*}
\end{proof}

\Tightness*
\begin{proof}
For any $\beta\in [0,1]$, we consider the distribution that
concentrates on a singleton $\curl*{x_0}$ and satisfies
$p(x_0,1)= \frac{1+\beta}{2}$, $p(x_0,2)= \frac{1-\beta}{2}$,
$p(x_0,y)=0,\,3\leq y\leq n$. We take $h_{\tau}\in \sH$ such that
$e^{h_{\tau}(x,1)}=e^{h_{\tau}(x,2)}$, $e^{h_{\tau}(x,y)}=0,\, 3\leq
y\leq n$.  \ignore{We take the hypotheses $\curl*{h_{\tau}:\tau
    \in[0,\infty)}\in \sH$ using the following criteria, when $\tau\in
       [0,1]$, take $h_{\tau}\in \sH$ such that
       $e^{h_{\tau}(x,1)}=e^{h_{\tau}(x,2)}$, $e^{h_{\tau}(x,y)}=0,\,
       3\leq y\leq n$; when $\tau\in (1,\plus\infty)$, take
       $h_{\tau}\in \sH$ such that
       $e^{h_{\tau}(x,1)}=e^{h_{\tau}(x,2)}= \cdots=e^{h_{\tau}(x,n-1)}=e^{h_{\tau}(x,n)}$.}
  Then,
\begin{align*}
\sR_{\ell_{0-1}}(h_{\tau})- \sR_{\ell_{0-1},\sH}^*+\sM_{\ell_{0-1},\sH} & = \sR_{\ell_{0-1}}(h_{\tau}) - \mathbb{E}_{X} \bracket* {\sC^*_{\ell_{0-1}}(\sH,x)}= \sC_{\ell_{0-1}}(h_{\tau},x_0) - \sC^*_{\ell_{0-1}}\paren*{\sH,x_0}= \beta
\end{align*}
and for any $\tau\in[0,1]$,
\begin{align*}
& \sR_{\ell_{\tau}^{\rm{comp}}}(h_{\tau}) - \sR_{\ell_{\tau}^{\rm{comp}}}^*(\sH)+ \sM_{\ell_{\tau}^{\rm{comp}}}(\sH)\\ & = \sR_{\ell_{\tau}^{\rm{comp}}}(h_{\tau}) - \mathbb{E}_{X} \bracket* {\sC^*_{\ell_{\tau}^{\rm{comp}}}(\sH,x)}\\
& = \sC_{\ell_{\tau}^{\mathrm{comp}}}(h_{\tau},x_0) - \sC^*_{\ell_{\tau}^{\mathrm{comp}}}\paren*{\sH,x_0}\\
& =p(x_0,1) \ell_{\tau}^{\rm{comp}}(h_{\tau}, x_0, 1) + p(x_0,2) \ell_{\tau}^{\rm{comp}}(h_{\tau}, x_0, 2) - \inf_{h\in \sH} \bracket*{p(x_0,1) \ell_{\tau}^{\rm{comp}}(h, x_0, 1) + p(x_0,2) \ell_{\tau}^{\rm{comp}}(h, x_0, 2)}\\
& = \sT_{\tau}(\beta), \tag{By \eqref{eq:comp-loss} and \eqref{eq:comp-sum-Cstar}}
\end{align*}
which completes the proof.
\end{proof}

\section{Approximations of \texorpdfstring{$\sT_{\tau}$}{T} and \texorpdfstring{$\Gamma_{\tau}$}{Gamma}}
\label{app:Gamma-upper-bound}
In this section, we show how $\sT_{\tau}$ can be lower bounded by its polynomial approximation $\wt \sT_{\tau}$, and accordingly, $\Gamma_{\tau}$ can then be upper bounded by $\wt \Gamma_{\tau}= \wt \sT_{\tau}^{-1}$. By analyzing the Taylor expansion, we obtain for any $\beta\in [-1,1]$,
\begin{equation}
\label{eq:inequality-auxiliary-1}
\begin{aligned}
\paren*{\frac{(1+\beta)^r+(1-\beta)^r}{2}}^{\frac1r} & \geq 1+\frac{\beta^2}{2}\paren*{1-\frac1r}, \text{ for all } r\geq 1\\
\paren*{\frac{(1+\beta)^r+(1-\beta)^r}{2}}^{\frac1r} & \leq 1 - \frac{\beta^2}{2}\paren*{1-r}, \text{ for all } \frac12\leq r\leq 1.
\end{aligned}
\end{equation}
and 
\begin{align}
\label{eq:inequality-auxiliary-2}
\frac{1+\beta}{2}\log\bracket*{1+\beta} + \frac{1-\beta}{2}\log\bracket*{1-\beta}\geq \frac{\beta^2}{2}.
\end{align}
For $\tau \in [0,1)$, we have
\begin{align*}
\sT_{\tau}(\beta) 
& = \frac{2^{1-\tau}}{1-\tau}\bracket*{1-\bracket*{\frac{\paren*{1+\beta}^{\frac1{2 - \tau }}+\paren*{1-\beta}^{\frac1{2 - \tau }}}{2}}^{2 - \tau }}\\
& \geq \frac{2^{1-\tau}}{1-\tau}\bracket*{1-\bracket*{1 - \frac{\beta^2}{2}\frac{1-\tau}{2-\tau}}} \tag{using \eqref{eq:inequality-auxiliary-1} with $r= \frac{1}{2-\tau}\in\bracket*{\frac12,1}$}\\
& = \frac{\beta^2}{2^{\tau}(2-\tau)}\\
& = \wt \sT_{\tau}(\beta).
\end{align*}
For $\tau\in (1,2)$, we have
\begin{align*}
\sT_{\tau}(\beta) 
& = \frac{1}{(\tau-1)n^{\tau-1}}\bracket*{\bracket*{\frac{\paren*{1+\beta}^{\frac1{2 - \tau }}+\paren*{1-\beta}^{\frac1{2 - \tau }}}{2}}^{2 - \tau }-1}\\
& \geq \frac{1}{(\tau-1)n^{\tau-1}}\bracket*{\bracket*{1 + \frac{\beta^2}{2}(\tau-1)}-1} \tag{using \eqref{eq:inequality-auxiliary-1} with $r= \frac{1}{2-\tau}\geq 1$}\\
& = \frac{\beta^2}{2n^{\tau-1}}\\
& = \wt \sT_{\tau}(\beta).
\end{align*}
For $\tau=1$, we have
\begin{align*}
\sT_{\tau}(\beta) & = \frac{1+\beta}{2}\log\bracket*{1+\beta} + \frac{1-\beta}{2}\log\bracket*{1-\beta}\\
&\geq \frac{\beta^2}{2} \tag{using \eqref{eq:inequality-auxiliary-2}}\\
& = \wt \sT_{\tau}(\beta).
\end{align*}
For $\tau\geq 2$,  $\sT_{\tau}(\beta)= \frac{\beta}{(\tau-1)n^{\tau-1}}= \wt \sT_{\tau}(\beta)$.
Therefore, for any $\tau\in [0,\plus\infty)$, 
\begin{align*}
\sT_{\tau}(\beta)\geq \wt\sT_{\tau}(\beta)=
\begin{cases}
\frac{\beta^2}{2^{\tau}(2-\tau)} & \tau \in [0,1)\\
\frac{\beta^2}{2n^{\tau-1}} & \tau \in [1,2)\\
\frac{\beta}{(\tau-1)n^{\tau-1}}& \tau \in [2,\plus \infty).
\end{cases}
\end{align*}
Furthermore, by using Taylor expansion, we have
\begin{align*}
\lim_{\beta \to 0^{+}}\frac{\wt\sT_{\tau}(\beta)}{\sT_{\tau}(\beta)} = c>0 \text{ for some constant $c>0$}.
\end{align*}
Thus, the order of polynomials $\wt\sT_{\tau}(\beta)$ is tightest.
Since
$\Gamma_{\tau}= \sT_{\tau}^{-1} $ and $\wt \Gamma_{\tau}=
\wt \sT_{\tau}^{-1}$, we also obtain for any $\tau\in [0,\plus\infty)$, 
\begin{align*}
\Gamma_{\tau}(t)\leq  \wt \Gamma_{\tau}(t)= \wt\sT_{\tau}^{-1}(t)= \begin{cases}
\sqrt{2^{\tau}(2-\tau) t} & \tau\in [0,1)\\
\sqrt{2n^{\tau-1} t } & \tau\in [1,2) \\
(\tau - 1) n^{\tau - 1} t & \tau \in [2,\plus \infty).
\end{cases}
\end{align*}

\section{Characterization of minimizability gaps
  (proofs of Theorem~\ref{Thm:gap-upper-bound} and
  Theorem~\ref{Thm:gap-upper-bound-determi})}
\label{app:gap-upper-bound}
\GapUpperBound*
\begin{proof}
Using the fact that $\Phi^{\tau}$ is concave and non-decreasing, by
\eqref{eq:concave-Phi1}, we can then upper bound the minimizability
gaps for different $\tau \geq 0$ as follows,
\begin{align*}
\sM_{\ell_{\tau}^{\rm{comp}}}(\sH)
\leq \Phi_{\tau}\paren*{\sR^*_{\ell_{\tau=0}^{\rm{comp}}}(\sH)} - \E_x[\sC^*_{\ell_{\tau}^{\rm{comp}}}(\sH, x)].
\end{align*}
By definition, the conditional $\ell_{\tau}^{\rm{comp}}$-risk can be expressed as follows: 
\begin{align}
\label{eq:cond-comp-sum}
\sC_{\ell_{\tau}^{\rm{comp}}}(h, x)  =  \sum_{y\in \sY} p(x,y) \Phi_{\tau}\paren*{\sum_{y'\neq y}\exp\paren*{h(x,y')-h(x,y)}}.
\end{align}
Note that $\sC_{\ell_{\tau}^{\rm{comp}}}(h, x)$ is convex and differentiable with respect to $h(x,y)$s, by taking the partial derivative and using the derivative of $\Phi_{\tau}$ given in \eqref{eq:Phi1-derivative}, we obtain
\begin{equation}
\label{eq:cond-comp-sum-deriv}
\begin{aligned}
&\frac{\partial \sC_{\ell_{\tau}^{\rm{comp}}}(h, x)}{\partial h(x,y)}\\
& =p(x,y)\frac{\partial \Phi_{\tau}}{\partial u}\paren*{\sum_{y'\neq y}\exp\paren*{h(x,y')-h(x,y)}}\paren*{-\sum_{y'\neq y}\exp\paren*{h(x,y')-h(x,y)}}\\
& + \sum_{y'\neq y} p(x,y')\frac{\partial \Phi_{\tau}}{\partial u}\paren*{\sum_{y''\neq y'}\exp\paren*{h(x,y'')-h(x,y')}}\paren*{\exp\paren*{h(x,y)-h(x,y')}}\\
& =p(x,y)\frac{-\sum_{y'\neq y}\exp\paren*{h(x,y')-h(x,y)}}{\bracket*{\sum_{y'\in \sY}\exp\paren*{h(x,y')-h(x,y)}}^{\tau}} + \sum_{y'\neq y} p(x,y')\frac{\exp\paren*{h(x,y)-h(x,y')}}{\bracket*{\sum_{y''\in \sY}\exp\paren*{h(x,y'')-h(x,y')}}^{\tau}}
\end{aligned}
\end{equation}
Let $\sS(x,y) = \sum_{y'\in \sY}\exp\paren*{h(x,y')-h(x,y)}$. Then,
$\exp\paren*{h(x,y)-h(x,y')} = \frac{\sS(x,y')}{\sS(x,y)}$ and thus
\eqref{eq:cond-comp-sum-deriv} can be written as
\begin{align}
\label{eq:cond-comp-sum-deriv-S}
\frac{\partial \sC_{\ell_{\tau}^{\rm{comp}}}(h, x)}{\partial h(x,y)} = p(x,y)\frac{-\sS(x,y)+1}{\sS(x,y)^{\tau}} + \sum_{y'\neq y} p(x,y') \frac{1}{\sS(x,y')^{\tau-1}\sS(x,y)}
\end{align}
It is straightforward to verify that 
\begin{align}
\label{eq:solution}
\sS^*(x,y) =
\begin{cases}
\frac{\sum_{y'\in \sY}p(x,y')^{\frac{1}{2-\tau}}}{p(x,y)^{\frac{1}{2-\tau}}} & \tau \neq 2\\
\frac{1}{\mathds{1}_{y= \argmax_{y'\in \sY}p(x,y')}} & \tau =2
\end{cases}
\end{align}
satisfy
\begin{align*}
\frac{\partial \sC_{\ell_{\tau}^{\rm{comp}}}(h, x)}{\partial h(x,y)} = 0, \forall y\in \sY.
\end{align*}
When $\sH$ is symmetric and complete, \eqref{eq:solution} can be
attained by some $h^*\in \sH$.  Since
$\sC_{\ell_{\tau}^{\rm{comp}}}(h, x)$ is convex and differentiable with
respect to $h(x,y)$s, we know that $h^*$ achieves the minimum of
$\sC_{\ell_{\tau}^{\rm{comp}}}(h, x)$ within $\sH$.  Then,
\begin{align*}
\sC^*_{\ell_{\tau}^{\rm{comp}}}(\sH, x) 
& =  \sC_{\ell_{\tau}^{\rm{comp}}}(h^*, x)\\
& = \sum_{y\in \sY} p(x,y) \Phi_{\tau}\paren*{\sum_{y'\neq y}\exp\paren*{h^*(x,y')-h^*(x,y)}}\\
& = \sum_{y\in \sY} p(x,y) \Phi_{\tau}\paren*{\sS^*(x,y)-1}\tag{by the def. of $\sS(x,y)$.}\\
& = \begin{cases}
\sum_{y\in \sY} p(x,y) \frac{1}{1 - \tau} \paren*{\sS^*(x,y)^{1 - \tau} - 1} & \tau \geq 0, \tau\neq1\\
\sum_{y\in \sY} p(x,y) \log\bracket*{\sS^*(x,y) } & \tau=1\\
\end{cases} \tag{by \eqref{eq:comp-loss}.}\\
&  = \begin{cases}
\frac{1}{1 - \tau} \paren*{\bracket*{\sum_{y\in \sY}p(x,y)^{\frac{1}{2-\tau}}}^{2 - \tau} - 1} & \tau\geq 0, \tau\neq1, \tau \neq 2\\
-\sum_{y\in \sY} p(x,y) \log\bracket*{p(x,y) } & \tau=1\\
1 - \max_{y\in \sY}p(x,y) & \tau =2.
\end{cases}
\tag{by \eqref{eq:solution}}
\end{align*}
\ignore{ Since $\Phi_{\tau}= \int_{0}^u \frac{1}{(1+t)^{\tau}}\,dt$ is
  decreasing with respect to $\tau\geq 0$, for any $h\in \sH$ and
  $x\in \sX$, $\sC_{\ell_{\tau}^{\rm{comp}}}(h, x)$ expressed in
  \eqref{eq:cond-comp-sum} is a decreasing function of $\tau \geq 0$,
  we obtain that $\sC^*_{\ell_{\tau}^{\rm{comp}}}(\sH, x)= \inf_{h\in
    \sH}\sC_{\ell_{\tau}^{\rm{comp}}}(h, x)$ is also a decreasing
  function of $\tau \geq 0$, which concludes the proof.}
\end{proof}

\GapUpperBoundDetermi*
\begin{proof}
Using the fact that $\Phi_{\tau}$ is concave and non-decreasing, by
\eqref{eq:concave-Phi1}, we can then upper bound the minimizability
gaps for different $\tau \geq 0$ as follows,
\begin{align*}
\sM_{\ell_{\tau}^{\rm{comp}}}(\sH)
\leq \Phi_{\tau}\paren*{\sR^*_{\ell_{\tau=0}^{\rm{comp}}}(\sH)} - \E_x[\sC^*_{\ell_{\tau}^{\rm{comp}}}(\sH, x)].
\end{align*}
Let $y_{\max} = \argmax p(x,y)$. By definition, for any deterministic distribution, the conditional $\ell_{\tau}^{\rm{comp}}$-risk can be expressed as follows:
\begin{equation}
\label{eq:cond-comp-sum-determi}
\begin{aligned}
\sC_{\ell_{\tau}^{\rm{comp}}}(h, x)  =  & \Phi_{\tau}\paren*{\sum_{y'\neq y_{\max}}\exp\paren*{h(x,y')-h(x,y_{\max})}}\\
& = 
\begin{cases}
\frac{1}{1 - \tau} \paren*{\paren*{1 + \frac{\sum_{y'\neq y_{\max}}\exp\paren*{h(x,y')}}{\exp(h(x,y_{\max}))}}^{1 - \tau} - 1} \\
\log\paren*{1 + \frac{\sum_{y'\neq y_{\max}}\exp\paren*{h(x,y')}}{\exp(h(x,y_{\max}))}}.
\end{cases}
\end{aligned}
\end{equation}
Since for any $\tau>0$, $\Phi_{\tau}$ is increasing, under the
assumption of $\sH$, $h^*$ that satisfies
\begin{align}
\label{eq:solution-determi}
h^*(x,y) = 
\begin{cases}
\Lambda & y = y_{\max}\\
-\Lambda & \text{otherwise}
\end{cases}
\end{align}
achieves the minimum of $\sC_{\ell_{\tau}^{\rm{comp}}}(h, x)$ within $\sH$.
Then, 
\begin{align*}
\sC^*_{\ell_{\tau}^{\rm{comp}}}(\sH, x) 
& =  \sC_{\ell_{\tau}^{\rm{comp}}}(h^*, x)\\
& = 
\begin{cases}
\frac{1}{1 - \tau} \paren*{\paren*{1 + \frac{\sum_{y'\neq y_{\max}}\exp\paren*{h^*(x,y')}}{\exp(h^*(x,y_{\max}))}}^{1 - \tau} - 1} \\
\log\paren*{1 + \frac{\sum_{y'\neq y_{\max}}\exp\paren*{h^*(x,y')}}{\exp(h^*(x,y_{\max}))}}.
\end{cases}\\
& = \begin{cases}
\frac{1}{1 - \tau} \paren*{\bracket*{1 + e^{-2 \Lambda}(n - 1)}^{1 - \tau} - 1} & \tau\geq 0, \tau\neq1\\
\log\bracket*{1 + e^{-2 \Lambda}(n - 1) } & \tau=1.
\end{cases}
\tag{by \eqref{eq:solution-determi}}
\end{align*}
Since $\sC^*_{\ell_{\tau}^{\rm{comp}}}(\sH, x) $ is independent of $x$, we have $\E_x[\sC^*_{\ell_{\tau}^{\rm{comp}}}(\sH, x)]= \sC^*_{\ell_{\tau}^{\rm{comp}}}(\sH, x)$ and thus
\begin{align*}
\sM_{\ell_{\tau}^{\rm{comp}}}(\sH)
\leq \Phi_{\tau}\paren*{\sR^*_{\ell_{\tau=0}^{\rm{comp}}}(\sH)} - \sC^*_{\ell_{\tau}^{\rm{comp}}}(\sH, x),
\end{align*}
which concludes the proof.
\end{proof}

\section{Proof of Lemma~\ref{lemma:lemma-compare}}
\label{app:lemma-compare}
\LemmaCompare*
\begin{proof}
For any $u_1 \geq u_2\geq 0$ and $\tau \neq 1$, we have
\begin{align*}
& \frac{\partial \paren*{\Phi_{\tau}(u_1)-\Phi_{\tau}(u_2)}}{\partial \tau}\\
& = \frac{\paren*{(1 + u_1)^{1 - \tau} - (1 + u_2)^{1 - \tau}}}{(1-\tau)^2} + \frac{1}{1-\tau}\paren*{\paren*{1 + u_2}^{1-\tau}\log(1+u_2)-\paren*{1 + u_1}^{1-\tau}\log(1+u_1)}\\
& = \frac{g(u_1,\tau) - g(u_2,\tau)}{(1-\tau)^2}
\end{align*}
where $g(t,\tau) =(1 + t)^{1 - \tau} - (1-\tau)\paren*{1 +
  t}^{1-\tau}\log(1+t)$. By taking the partial derivative, we obtain
for any $\tau \neq 1$ and $t\geq 0$,
\begin{align*}
\frac{\partial g}{\partial t} 
 = -(1-\tau)^2(1+t)^{\tau}\log (1+t)\leq  0
\end{align*}
Therefore, for any $u_1 \geq u_2\geq 0$ and $\tau \neq 1$,
$g(u_1,\tau) \leq  g(u_2,\tau)$ and
\begin{align*}
\frac{\partial \paren*{\Phi_{\tau}(u_1)-\Phi_{\tau}(u_2)}}{\partial \tau} \leq 0,
\end{align*}
which implies that for any $u_1 \geq u_2\geq 0$ and $\tau\neq 1$,
$\Phi_{\tau}(u_1)-\Phi_{\tau}(u_2)$ is a non-increasing function of
$\tau$.  Moreover, since for $x\geq 1$, $\frac{1}{\tau-1}
\paren*{x^{\tau-1} - 1} \to \log(x)$ as $\tau \to 1$, we know that for
any $u_1 \geq u_2\geq 0$, $\Phi_{\tau}(u_1)-\Phi_{\tau}(u_2)$ is
continuous with respect to $\tau=1$. Therefore, we conclude that for
any $u_1 \geq u_2\geq 0$, $\Phi_{\tau}(u_1)-\Phi_{\tau}(u_2)$ is
non-increasing with respect to $\tau$.
\end{proof}

\section{Proof of adversarial \texorpdfstring{$\sH$}{H}-consistency bound for adversarial comp-sum losses (Theorem~\ref{Thm:bound_comp_rho_adv})}
\label{app:deferred_proofs_adv_comp}

\BoundCompRhoAdv*
\begin{proof}
Let $\ov \sH_\gamma(x) = \curl*{h\in\sH:\inf_{x':\norm*{x-x'}\leq
    \gamma}\rho_h(x', \hh(x)) > 0}$ and $p(x)=(p(x, 1), \ldots, p(x,
c))$.  For any $x\in \sX$ and $h\in \sH$, we define $h\paren*{x,
  \curl*{1}^h_x}, h\paren*{x, \curl*{2}^h_x},\ldots, h\paren*{x,
  \curl*{c}^h_x}$ by sorting the scores $\curl*{h(x, y):y\in \sY}$ in
increasing order, and $p_{\bracket*{1}}(x), p_{\bracket*{2}}(x),
\ldots, p_{\bracket*{c}}(x)$ by sorting the probabilities $\curl*{p(x,
  y):y\in \sY}$ in increasing order.  Note
$\curl*{c}^h_x= \hh(x)$. Since $\sH$ is symmetric and locally
$\rho$-consistent, for any $x\in \sX$, there exists a hypothesis $h^*
\in \sH$ such that
\begin{align*}
& \inf_{x'\colon \norm*{x - x'}\leq \gamma}\abs*{h^*(x', i) - h^*(x', j)}  \geq \rho, \forall i\neq j \in \sY\\
& p(x,\curl*{k}^{h^*}_{x'})  =p_{\bracket*{k}}(x), \forall x'\in \curl*{x'\colon \norm*{x - x'}\leq \gamma}, \forall k\in
\sY.
\end{align*}
Then, we have
\begin{align*}
&\sC^*_{\wt \ell^{\mathrm{comp}}_{\tau,\rho}}(\sH,x)\\
& \leq \sC_{\wt \ell^{\mathrm{comp}}_{\tau,\rho}}(h^*,x)\nonumber\\
& = \sum_{y\in \sY}  \sup_{x':\norm*{x-x'}\leq \gamma}p(x, y)\Phi^{\tau}\paren*{ \sum_{y'\neq y} \Phi_{\rho}\paren*{h^*(x',y)-h^*(x',y')}}\nonumber\\
& =
\sum_{i=1}^c \sup_{x':\norm*{x-x'}\leq \gamma} p(x,\curl*{i}^{h^*}_{x'})\Phi^{\tau} \bracket*{\sum_{j=1}^{i-1}\Phi_{\rho}\paren*{h^*(x',\curl*{i}^{h^*}_{x'})-h^*(x',\curl*{j}^{h^*}_{x'})}+ \sum_{j=i+1}^{c}\Phi_{\rho}\paren*{h^*(x',\curl*{i}^{h^*}_{x'})-h^*(x',\curl*{j}^{h^*}_{x'})}}\\
& = \sum_{i=1}^c \sup_{x':\norm*{x-x'}\leq \gamma} p(x,\curl*{i}^{h^*}_{x'})\Phi^{\tau} \bracket*{\sum_{j=1}^{i-1}\Phi_{\rho}\paren*{h^*(x',\curl*{i}^{h^*}_{x'})-h^*(x',\curl*{j}^{h^*}_{x'})} + c-i} \tag{$\Phi_{\rho}(t)=1,\,\forall t\leq 0$}\\
& = \sum_{i=1}^c \sup_{x':\norm*{x-x'}\leq \gamma} p(x,\curl*{i}^{h^*}_{x'})\Phi^{\tau} (c-i) \tag{$\inf_{x':\norm*{x-x'}\leq \gamma}\abs*{h^*(x',i)-h^*(x',j)}\geq \rho$ for any $i\neq j$ and $\Phi_{\rho}(t)=0,\,\forall t\geq \rho$}\\
& = \sum_{i=1}^c p_{\bracket*{i}}(x) \Phi^{\tau}(c-i) \tag{$p(x,\curl*{k}^{h^*}_{x'})  =p_{\bracket*{k}}(x), \forall x'\in \curl*{x'\colon \norm*{x - x'}\leq \gamma}, \forall k\in
\sY$}.
\end{align*}

Note $\ov \sH_{\gamma}(x)\neq\emptyset$ under the assumption. Then, use the derivation above, we obtain
\begin{align*}
& \Delta\sC_{\wt \ell^{\mathrm{comp}}_{\tau,\rho},\sH}(h, x)\\
& =  \sum_{i=1}^c \sup_{x':\norm*{x-x'}\leq \gamma} p(x,\curl*{i}^{h}_{x'})\Phi^{\tau} \bracket*{\sum_{j=1}^{i-1}\Phi_{\rho}\paren*{h(x',\curl*{i}^{h}_{x'})-h(x',\curl*{j}^{h}_{x'})} + c-i} -  \sum_{i=1}^c p_{\bracket*{i}}(x) \Phi^{\tau}(c-i)\\
&\geq \Phi^{\tau}(1)\, p(x, \hh(x))\mathds{1}_{h\not \in \ov \sH_{\gamma}(x)} +  \sum_{i=1}^c \sup_{x':\norm*{x-x'}\leq \gamma} p(x,\curl*{i}^{h}_{x'}) \Phi^{\tau}\paren*{c-i} -  \sum_{i=1}^c p_{\bracket*{i}}(x) \Phi^{\tau}(c-i) \tag{$\Phi_{\rho}$ is non-negative and $\Phi^{\tau}$ is non-decreasing}\\
& \geq \Phi^{\tau}(1)\,p(x, \hh(x))\mathds{1}_{h\not \in \ov \sH_{\gamma}(x)} + \sum_{i=1}^c p(x,\curl*{i}^{h}_{x}) \Phi^{\tau}(c-i) - \sum_{i=1}^c p_{\bracket*{i}}(x) \Phi^{\tau}(c-i) \tag{$\sup_{x':\norm*{x-x'}\leq \gamma} p(x,\curl*{i}^{h}_{x'})\geq p(x,\curl*{i}^{h}_{x} $}\\
& = \Phi^{\tau}(1)\,p(x, \hh(x))\mathds{1}_{h\not \in \ov \sH_{\gamma}(x)} +\Phi^{\tau}(1) \paren*{\max_{y\in \sY}p(x, y) - p(x, \hh(x))} \\
&\qquad+ 
\begin{bmatrix}
\Phi^{\tau}(1)\\
\Phi^{\tau}(1)\\
\Phi^{\tau}(2)\\
\vdots\\
\Phi^{\tau}(c-1)
\end{bmatrix}
\cdot
\begin{bmatrix}
p(x,\curl*{c}^{h}_{x})\\
p(x,\curl*{c-1}^{h}_{x})\\
p(x,\curl*{c-2}^{h}_{x})\\
\vdots\\
p(x,\curl*{1}^{h}_{x})\\
\end{bmatrix}
-\begin{bmatrix}
\Phi^{\tau}(1)\\
\Phi^{\tau}(1)\\
\Phi^{\tau}(2)\\
\vdots\\
\Phi^{\tau}(c-1)
\end{bmatrix}
\cdot
\begin{bmatrix}
p_{\bracket*{c}}(x)\\
p_{\bracket*{c-1}}(x)\\
p_{\bracket*{c-2}}(x)\\
\vdots\\
p_{\bracket*{1}}(x)\\
\end{bmatrix}
\tag{$p_{[c]}(x)= \max_{y\in \sY}p(x, y)$, $\curl*{c}^{h}_x = \hh(x)$ and $\Phi^{\tau}(0)=0$}\\
& \geq \Phi^{\tau}(1)\, p(x, \hh(x))\mathds{1}_{h\not \in \ov \sH_{\gamma}(x)} + \Phi^{\tau}(1) \paren*{\max_{y\in \sY}p(x, y) - p(x, \hh(x))} \tag{ rearrangement inequality for $\Phi^{\tau}(1)\leq \Phi^{\tau}(1) \leq \Phi^{\tau}(2)\leq \cdots\leq \Phi^{\tau}(c-1)$ and $p_{\bracket*{c}}(x)\geq \cdots \geq  p_{\bracket*{1}}(x)$}\\
& = \Phi^{\tau}(1)\paren*{\max_{y\in \sY}p(x, y) - p(x, \hh(x))\mathds{1}_{h\in \ov \sH_{\gamma}(x)}}
\end{align*}
for any $h\in\sH$. Since $\sH$ is symmetric and $\ov \sH_{\gamma}(x)\neq\emptyset$, we have 
\begin{align*}
\Delta\sC_{\ell_{\gamma},\sH}(h, x)
& = \sC_{\ell_{\gamma}}(h, x)-\sC^*_{\ell_{\gamma}}(\sH,x)\\
& = \sum_{y\in \sY} p(x,y) \sup_{x':\norm*{x-x'}\leq \gamma}\mathds{1}_{\rho_h(x', y) \leq 0}-\inf_{h\in \sH}\sum_{y\in \sY} p(x,y) \sup_{x':\norm*{x-x'}\leq \gamma}\mathds{1}_{\rho_h(x', y) \leq 0}\\
& = \paren*{1-p(x,\hh(x))}\mathds{1}_{h\in \ov \sH_{\gamma}(x)}+\mathds{1}_{h\not \in \ov \sH_{\gamma}(x)} - \inf_{h\in \sH}\bracket*{\paren*{1-p(x,\hh(x))}\mathds{1}_{h\in \ov \sH_{\gamma}(x)}+ \mathds{1}_{h\not \in \ov \sH_{\gamma}(x)}}\\
& = \paren*{1-p(x,\hh(x))}\mathds{1}_{h\in \ov \sH_{\gamma}(x)} +\mathds{1}_{h\not \in \ov \sH_{\gamma}(x)} - \paren*{1-\max_{y\in \sY} p(x,y)} \tag{$\sH$ is symmetric and $\ov \sH_{\gamma}(x)\neq\emptyset$}\\
& = \max_{y\in \sY}p(x, y) - p(x, \hh(x))\mathds{1}_{h\in \ov \sH_{\gamma}(x)}.
\end{align*}
Therefore, by the definition, we obtain
\begin{align*}
\sR_{\ell_{\gamma}}(h)- \sR^*_{\ell_{\gamma}}(\sH)+\sM_{\ell_{\gamma}}(\sH)
& =
\mathbb{E}_{X}\bracket*{\Delta\sC_{\ell_{\gamma}}(h, x)}\\
& = \mathbb{E}_{X}\bracket*{\max_{y\in \sY}p(x, y) - p(x, \hh(x))\mathds{1}_{h\in \ov \sH_{\gamma}(x)}}\\
&\leq
\Phi^{\tau}(1)\,\mathbb{E}_{X}\bracket*{\Delta\sC_{\wt \ell^{\mathrm{comp}}_{\tau,\rho},\sH}(h, x)}\\
& = \Phi^{\tau}(1) \paren*{\sR_{\wt \ell^{\mathrm{comp}}_{\tau,\rho}}(h)-\sR^*_{\wt \ell^{\mathrm{comp}}_{\tau,\rho}}(\sH) + \sM_{\wt \ell^{\mathrm{comp}}_{\tau,\rho}}(\sH)},
\end{align*}
which implies that
\begin{align*}
\sR_{\ell_{\gamma}}(h)- \sR^*_{\ell_{\gamma}}(\sH) \leq \Phi^{\tau}(1)\paren*{\sR_{\wt \ell^{\mathrm{comp}}_{\tau,\rho}}(h)-\sR^*_{\wt \ell^{\mathrm{comp}}_{\tau,\rho}}(\sH) + \sM_{\wt \ell^{\mathrm{comp}}_{\tau,\rho}}(\sH)} - \sM_{\ell_{\gamma}}(\sH).
\end{align*}
\end{proof}


\section{Learning bounds (proof of Theorem~\ref{th:genbound})}
\label{app:genbound}

\GenBound*
\begin{proof}
  By the standard Rademacher complexity bounds \citep{MohriRostamizadehTalwalkar2018}, the following holds
  with probability at least $1 - \delta$ for all $h \in \sH$:
\[
\abs*{\sR_{\ell_{\tau}^{\rm{comp}}}(h) - \h \sR_{\ell_{\tau}^{\rm{comp}}, S}(h)}
\leq 2 \Rad_m^\tau(\sH) +
B_\tau \sqrt{\tfrac{\log (2/\delta)}{2m}}.
\]
Fix $\e > 0$. By the definition of the infimum, there exists $h^* \in
\sH$ such that $\sR_{\ell_{\tau}^{\rm{comp}}}(h^*) \leq
\sR_{\ell_{\tau}^{\rm{comp}}}^*(\sH) + \e$. By definition of
$\h h_S$, we have
\begin{align*}
  & \sR_{\ell_{\tau}^{\rm{comp}}}(\h h_S) - \sR_{\ell_{\tau}^{\rm{comp}}}^*(\sH)\\
  & = \sR_{\ell_{\tau}^{\rm{comp}}}(\h h_S) - \h \sR_{\ell_{\tau}^{\rm{comp}}, S}(\h h_S) + \h \sR_{\ell_{\tau}^{\rm{comp}}, S}(\h h_S) - \sR_{\ell_{\tau}^{\rm{comp}}}^*(\sH)\\
  & \leq \sR_{\ell_{\tau}^{\rm{comp}}}(\h h_S) - \h \sR_{\ell_{\tau}^{\rm{comp}}, S}(\h h_S) + \h \sR_{\ell_{\tau}^{\rm{comp}}, S}(h^*) - \sR_{\ell_{\tau}^{\rm{comp}}}^*(\sH)\\
  & \leq \sR_{\ell_{\tau}^{\rm{comp}}}(\h h_S) - \h \sR_{\ell_{\tau}^{\rm{comp}}, S}(\h h_S) + \h \sR_{\ell_{\tau}^{\rm{comp}}, S}(h^*) - \sR_{\ell_{\tau}^{\rm{comp}}}^*(h^*) + \e\\
  & \leq
  2 \bracket*{2 \Rad_m^\tau(\sH) +
B_\tau \sqrt{\tfrac{\log (2/\delta)}{2m}}} + \e.
\end{align*}
Since the inequality holds for all $\e > 0$, it implies:
\[
\sR_{\ell_{\tau}^{\rm{comp}}}(\h h_S) - \sR_{\ell_{\tau}^{\rm{comp}}}^*(\sH)
\leq 
4 \Rad_m^\tau(\sH) +
2 B_\tau \sqrt{\tfrac{\log (2/\delta)}{2m}}.
\]
Plugging in this inequality in the bound of
Theorem~\ref{Thm:bound_comp_sum} completes the proof.
\end{proof}

\end{document}